\documentclass[conference]{IEEEtran}

\IEEEoverridecommandlockouts                  
\usepackage[dvips]{graphicx}
\usepackage[english]{babel}
\usepackage{epsfig,multirow,inputenc}
\usepackage{amsmath,amssymb,bm}
\usepackage{amsfonts,dsfont,color,bbm}
\usepackage{algorithm,algorithmic}
\usepackage{epstopdf,epsf,epsfig}
\usepackage{cite}
\usepackage{geometry}
\usepackage{flexisym}
\usepackage{tabularx,ragged2e,booktabs}
\usepackage{amsthm}
\newtheorem{theorem}{Theorem}
\newtheorem{corollary}{Corollary}
\newtheorem{lemma}{Lemma}

\newtheorem{definition}{Definition}

\allowdisplaybreaks %breaks long equations if it does not fit into one page.\usepackage{amsthm}

\usepackage{graphics} % for pdf, bitmapped graphics files
\usepackage{epsfig} % for postscript graphics files
\usepackage{times} % assumes new font selection scheme installed
\usepackage{amsmath} % assumes amsmath package installed
\usepackage{amssymb}  % assumes amsmath package installed

\usepackage[final]{pdfpages}

\DeclareMathOperator*{\argmax}{arg\,max}

\newcommand{\sgn}{\mathop{\mathrm{sign}}}

\newcommand{\comment}[1]{}

\ifodd 0
\newcommand{\com}[1]{{\color{red}#1}}
\else
\newcommand{\com}[1]{}
\fi

\ifodd 1
\newcommand{\add}[1]{{\color{blue}#1}}
\else
\newcommand{\add}[1]{}
\fi

\ifodd 0
\newcommand{\rev}[1]{{\color{blue}#1}}
\else
\newcommand{\rev}[1]{#1}
\fi

\begin{document}
	
	\title{\LARGE \bf Gambler's Ruin Bandit Problem}
	
	\author{\IEEEauthorblockN{Nima Akbarzadeh, Cem Tekin} 
		\IEEEauthorblockA{Bilkent University, Electrical and Electronics Engineering Department, Ankara, Turkey}\thanks{Cem Tekin is supported by TUBITAK 2232 Fellowship (116C043).} }
	\maketitle
	
	\begin{abstract}
		In this paper, we propose a new multi-armed bandit problem called the {\em Gambler's Ruin Bandit Problem} (GRBP). In the GRBP, the learner proceeds in a sequence of rounds, where each round is a Markov Decision Process (MDP) with two actions (arms): a {\em continuation action} that moves the learner randomly over the state space around the current state; and a {\em terminal action} that moves the learner directly into one of the two terminal states (goal and dead-end state). The current round ends when a terminal state is reached, and the learner incurs a positive reward only when the goal state is reached.
		The objective of the learner is to maximize its long-term reward (expected number of times the goal state is reached), without having any prior knowledge on the state transition probabilities. 
		We first prove a result on the form of the optimal policy for the GRBP. Then, we define the regret of the learner with respect to an {\em omnipotent oracle}, which acts optimally in each round, and prove that it increases logarithmically over rounds.
		We also identify a condition under which the learner's regret is bounded. 
		A potential application of the GRBP is optimal medical treatment assignment, 
		in which the continuation action corresponds to a conservative treatment and the terminal action corresponds to a risky treatment such as surgery.
	\end{abstract}
	\section{Introduction} \label{sec:Introduction}
	Multi-armed bandits (MAB) are used to model a plethora of applications that require sequential decision making under uncertainty ranging from clinical trials \cite{villar2015mabcln} to web advertising \cite{tekin2015releaf}. In the conventional MAB \cite{lai1985aeaar, auer2002gdmdp} the learner chooses an action from a finite set of actions at each round, and receives a random reward. The goal of the learner is to maximize its long-term expected reward by choosing actions that yield high rewards. This is a non-trivial task, since the reward distributions are not known beforehand. Numerous order-optimal index-based learning rules have been developed for the conventional MAB \cite{auer2002gdmdp,garivier2011klucb,auer2010ucb}. These rules act myopically by choosing the action with the maximum index in each round. 
	
	Situations that require multiple actions to be taken in each round cannot be modeled using conventional MAB. As an example, consider medical treatment administration.
	At the beginning of each round a patient arrives to the {\em intensive care unit} (ICU) with a random initial health state. The goal state is defined as {\em discharge} and dead-end state is defined as {\em death}. Actions correspond to treatment options that move the patient randomly over the state space. The objective is to maximize the expected number of patients that are discharged by learning the optimal treatment policy using the observations gathered from the previous patients. In the example given above, each round corresponds to a goal-oriented Markov Decision Process (MDP) with dead-ends \cite{Kolobov12MDPGD}. The learner knows the state space, goal and dead-end states, but does not know the state transition probabilities a priori. At each round, the learner chooses a sequence of actions and only observes the state transitions that result from the chosen actions. In the literature, this kind of feedback information is called {\em bandit feedback} \cite{bubeck2012regmab}. 
	
	Motivated by the application described above, we propose a new MAB problem in which multiple arms are selected in each round until a terminal state is reached. Due to its resemblance to the {\em Gambler's Ruin Problem} \cite{takacs1969crp,hunter2008grcw,uem2013mmmgrp}, we call this new MAB problem the {\em Gambler's Ruin Bandit Problem} (GRBP).
	In GRBP, the system proceeds in a sequence of rounds $\rho \in \{1,2,\ldots\}$. 
	Each round is modeled as an MDP (as in Fig. \ref{fig:GRBP} \com{(In Figure 1, $p^d$ and $p^u$ should change to $p^D$ and $p^C$.)} ) with unknown state transition probabilities and terminal (absorbing) states. The set of terminal states includes a {\em goal state} $G$ and a {\em dead-end state} $D$, and the non-terminal states are ordered between the goal and dead-end states. In each non-terminal state, there are two possible actions: a {\em continuation action} (action $C$) that moves the learner randomly over the state space around the current state; and a {\em terminal action} (action $F$) that moves the learner directly into a terminal state.
	Starting from a random, non-terminal initial state, the learner chooses a sequence of actions and observes the resulting state transitions until a terminal state is reached.
	The learner incurs a unit reward if the goal state is reached. Otherwise, it incurs no reward. The goal of the learner is to maximize its cumulative expected reward over the rounds. 
	
	If the state transition probabilities were known beforehand, an {\em omnipotent oracle} with unlimited computational power could calculate the optimal policy that maximizes the probability of hitting the goal state from any initial state, and then select its actions according to the optimal policy. We define the regret of the learner by round $\rho$ as the difference in the expected number of times the goal state is reached by the omnipotent oracle and the learner by round $\rho$. 
	
	First, we show that the optimal policy for GRBP can be computed in a straightforward manner: there exists a threshold state above which it is always optimal to take action $C$ and on or below which it is always optimal to take action $F$. 
	Then, we propose an online learning algorithm for the learner, and bound its regret for two different regions that the actual state transition probabilities can lie in. The regret is bounded (finite) in one region, while it is logarithmic in the number of rounds in the other region. 
	These bounds are problem-specific, in the sense that they are functions of the state transition probabilities.
	Finally, we illustrate the behavior
	  of the regret as a function of the state transition probabilities through numerical experiments.

	\begin{figure}[t]
		\centering
		\includegraphics[scale=0.4]{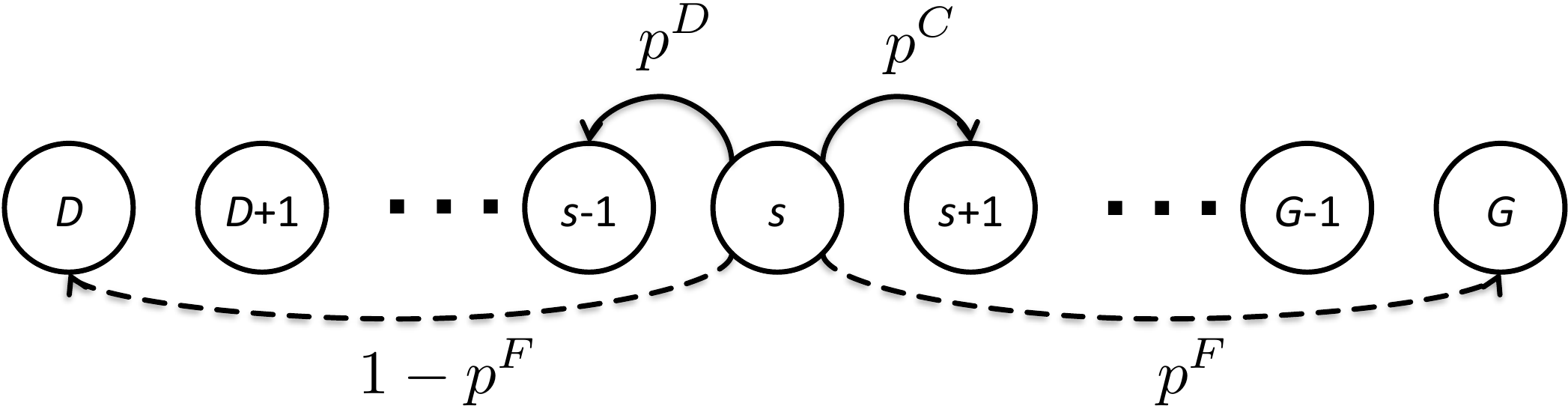}
		\caption{State transition model of the GRBP. Only state transitions out of state $s$ are shown. Dashed arrows correspond to possible state transitions by taking action $F$, while solid arrows correspond to possible state transitions by taking action $C$. Weights on the arrows correspond to state transition probabilities. The state transition probabilities for all other non-terminal states are the same as state $s$.} \label{fig:GRBP}
	\end{figure} 
	%Due to the way the continuation action is defined, this MDP resembles the gambler's ruin problem \cite{takacs1969crp,edwards1983ppgr,neuts1984grmdt,katriel2014grdp,el-shehawey2009grfmc,coad2013gpsc,boots200irp,tsay2006sgrp,hunter2008grcw,harik1999gasp,uem2013mmmgrp,dawson2000eis}. Unlike our model, in a conventional gambler's ruin problem there is no terminal action. Hence, it can be modeled as a Markov chain. The gambler (which we refer to as the learner) starts from an arbitrary state between $D$ and $G$, and in each time slot it will either go one state to the right (up) with probability $p^{u}$ or one state left (down) with probability $p^{\text{d}}$.\footnote{In the fair case we have $p^{u} = p^{\text{d}} = 0.5$.} Many variants of the gambler's ruin problem are proposed and widely used in practical applications including business management (price models) \cite{coad2013gpsc,boots200irp}, computer science (genetic algorithms) \cite{harik1999gasp} and physics (quantum motion) \cite{dawson2000eis}. 
	
	The contributions of this paper can be summarized as follows:
	\begin{itemize}
		\item We define a new MAB problem, called GRBP, in which the learner takes a sequence of actions in each round with the objective of reaching to the goal state.
		\item We show that using conventional MAB algorithms such as UCB1 \cite{auer2002gdmdp} in GRBP by enumerating all deterministic Markov policies is very inefficient and results in high regret. 
		\item We prove that the optimal policy for GRBP has a threshold form and the value of the threshold can be calculated in a computationally efficient way.
		\item We derive bounds on the regret of the learner with respect to an omnipotent oracle that acts optimally. Unlike conventional MAB where the regret growth is at least logarithmic in the number of rounds \cite{lai1985aeaar}, in GRBP regret can be either logarithmic or bounded, based on the values of the state transition probabilities. We explicitly characterize the region of state transition probabilities in which the regret is bounded.
	\end{itemize}
	Remainder of the paper is organized as follows. Related work is given in Section \ref{sec:related}.
	GRBP is defined in Section \ref{sec:problem}. Form of the optimal policy for the GRBP is given in Section \ref{sec:structure}. The learning algorithm for GRBP is given in Section \ref{sec:algorithm} together with its regret analysis. 
	%Regret bounds of GETBE are explained in Section \ref{sec:regret}. 
	Numerical results are shown in Section \ref{sec:numeric}. Conclusion is given in Section \ref{sec:Conclusion}. 
	%	\rev{Due to limited space, proofs are given in our online appendix \cite{akztkn2016grbp}.}
	\section{Related Work} \label{sec:related} 
	\subsection{Gambler's Ruin Problem}
	If action $F$ is removed from the GRBP, it becomes the Gambler's Ruin Problem. 
	In the model of Hunter et al. \cite{hunter2008grcw} of the Gambler's Ruin Problem, in addition to the standard outcome of moving one state to the left or right, two extra outcomes are also considered. One outcome changes the state immediately to $G$, while the other outcome changes the state immediately to $D$. These outcomes are referred to as {\em Windfall} and {\em Catastrophe} outcomes, respectively. 
	The ruin and winning probabilities and the duration of the game are calculated based on these additional outcomes. 
	In another model \cite{uem2013mmmgrp}, modifications such as the chance of absorption in states other than $G$ and $D$ and staying in the same state are considered. The ruin and winning probabilities are calculated according to the proposed state transition model. 
	Unlike GRBP which is an MDP, the Gambler's Ruin Problem is a Markov chain. Moreover, the ruin and winning probabilities in the models above can be calculated exactly since the transition probabilities are assumed to be known.

	\subsection{MDPs}
	GRBP is closely related to goal oriented MDPs and stochastic shortest path problems \cite{bertsekas1995DP}. For these problems, in each state (or time epoch), an action has to be taken with the aim of reaching to the goal state ($G$) with minimum cost. 
	For this task, the optimal policy have to be determined beforehand using the set of known transition probabilities.
	Recently, progress has been made in obtaining solutions for MDPs that have dead-end ($D$) states in addition to goal ($G$) states \cite{teichteil2012sssp, Kolobov12MDPGD}. These solutions require value iteration and heuristic search methods to be performed using the  knowledge of transition probabilities. To the best of our knowledge, a reinforcement learning algorithm that works without knowing the transition probabilities a priori and that achieves logarithmic regret bounds, has not been developed yet for these problems. 
	
	Reinforcement learning in MDPs is considered by numerous researchers \cite{tewari2008, auer2009near}. In these works, it is assumed that the underlying MDP is unknown but ergodic, i.e., it is possible to reach from any state to all other states with a positive probability under any policy. These works adopt the principle of optimism under uncertainty to choose an action that maximizes the expected reward among a set of MDP models that are consistent with the estimated transition probabilities. Unlike these works, in GRBP (i) the MDP is not ergodic, and (ii) the reward is obtained only in the terminal state and not after each chosen action.
	
	\subsection{Multi-armed Bandits} 
	
	%MAB problem is an online learning problem in which the learner should make a decision on whether to explore an arm (action) or exploit it based on its past observations. Exploration assists the learner to figure out the rewards of its actions. On the other hand, exploitation helps the learner to maximize its reward. The learner's total reward is maximized by balancing exploration and exploitation.
	Over the last decade many variations of the MAB problem is studied and many different learning algorithms are proposed, including Gittins index \cite{gittins1974isdexp}, upper confidence bound policies (UCB-1, UCB-2, Normalized UCB, KL-UCB) \cite{auer2002gdmdp,auer2010ucb,garivier2011klucb}, greedy policies ($\epsilon$-greedy algorithm) \cite{auer2002gdmdp} and Thompson sampling \cite{thompson2012tsmab} (see \cite{bubeck2012regmab} for a comprehensive analysis of the MAB problem). The performance of a learning algorithm for a MAB problem is computed using the notion of regret. For the stochastic MAB problem \cite{lai1985aeaar}, the regret is defined as the difference between the total (expected) reward of the learning algorithm and an {\em oracle} which acts optimally based on complete knowledge of the problem parameters. It is shown that the regret grows logarithmically in the number of rounds for this problem. 
	
	GRBP can be viewed as a MAB problem in which each arm corresponds to a policy. Since the set of possible deterministic policies for the GRBP is exponential in the number of states, it is infeasible to use algorithms developed for MAB problems to directly learn the optimal policy by experimenting with different policies over different rounds. In addition, GRBP model does not fit into the combinatorial models proposed in prior works \cite{cesa2012combinatorial}. Due to these differences, existing MAB solutions cannot solve GRBP in an efficient way. Therefore, a new learning methodology that exploits the structure of the GRBP is needed.
	\section{Problem Formulation} \label{sec:problem}
	\subsection{Definition of the GRBP}
	In the GRBP, the system is composed of a finite set of states ${\cal S} := \{ D , 1, \ldots, G \}$, where integer $D = 0$ denotes the {\em dead-end} state and $G$ denotes the {\em goal} state. 
	The set of {\em initial} (starting) states is denoted by $\tilde{ {\cal S} } := \{ 1, \ldots, G-1 \}$.
	The system operates in rounds ($\rho = 1,2,\ldots$).
	The initial state of each round is drawn from a probability distribution $q(s), s \in \tilde{{\cal S}}$ over the set of initial states $\tilde{ {\cal S} }$, \rev{such that $1-q(1)>0$}.
	The current round ends and the next round starts when the learner hits state $D$ or $G$. 
	Because of this, $D$ and $G$ are called {\em terminal states}. All other states are called non-terminal states.
	Each round is divided into multiple time slots in which the learner takes an action in each time slot from the action set ${\cal A} := \{ C, F \}$ with the aim of reaching to state $G$.
	Here, $C$ denotes the continuation action and $F$ is the terminal action.
	According to Fig. \ref{fig:GRBP}, action $C$ moves the learner one state to the right or to the left of the current state.
	Action $F$ moves the learner directly to one of the terminal states.
	Possible outcomes of each action in a non-terminal state $s$ is shown in Fig. \ref{fig:GRBP}. 
	Let $s^\rho_t$ denote the state at the beginning of the $t$th time slot of round $\rho$ and \rev{$a^{\rho}_t$ denote the action taken at the $t$th time slot of round $\rho$.}
	The state transition probabilities for action $C$ are given by 
	\begin{align}
		& \Pr ( s^\rho_{t+1} = s+1 | s^\rho_{t} = s, \rev{a^{\rho}_t = C}) = p^C, ~ t \geq 1, ~ s \in \tilde{ {\cal S} } \nonumber \\
		& \Pr ( s^\rho_{t+1} = s-1 | s^\rho_{t} = s, \rev{a^{\rho}_t = C} ) = p^D, ~ t \geq 1, ~ s \in \tilde{ {\cal S} } \nonumber
	\end{align} 
	where $ p^C + p^D = 1$. 
	The state transition probabilities for action $F$ are given by 
	\begin{align}
		& \Pr ( s^\rho_{t+1} = G | s^\rho_{t} = s, \rev{a^{\rho}_t = F}  ) =  p^{F}, ~ t \geq 1, ~ s \in \tilde{ {\cal S} }  \nonumber \\
		& \Pr ( s^\rho_{t+1} = D | s^\rho_{t} = s, \rev{a^{\rho}_t = F}  ) = 1 - p^{F}, ~ t \geq 1, ~ s \in \tilde{ {\cal S} } \nonumber
	\end{align} 
	where $0 < p^{F} < 1$. If the state transition probabilities are known, each round can be modeled as a MDP and an optimal policy can be found by dynamic programming \cite{bellman1957DP,bertsekas1995DP}. 
	\subsection{Value Functions, Rewards and the Optimal Policy}
	Let $\pi = (\pi_1, \pi_2, \ldots)$, where $\pi_t : \tilde{ {\cal S} }   \rightarrow {\cal A}$, $t \geq 1$ represent a deterministic Markov policy. $\pi$ is a stationary policy if $\pi_t = \pi_{t'}$ for all $t$ and $t'$. For this case we will simply use $\pi : \tilde{ {\cal S} }   \rightarrow {\cal A}$ to denote a stationary deterministic Markov policy.
	Since the time horizon is infinite within a round and the state transition probabilities are time-invariant, it is sufficient to search for the optimal policy within the set of stationary deterministic Markov policies, which is denoted by $\Pi$.
	Let $V^{\pi}(s)$ denote the probability of reaching to $G$ by using policy $\pi$ given that the system is in state $s$. Let $Q^{\pi}(s,a)$ denote the probability of reaching to $G$ by taking action $a$ in state $s$, and then continuing according to policy $\pi$.
	We have
	\begin{align}
		& Q^{\pi}(s, C ) = p^C V^{\pi}( s+1 ) + p^D V^{\pi}( s-1 ), \notag \\
		& Q^{\pi}(s, F ) = p^{F} \notag
	\end{align}
	for $s \in \tilde{ {\cal S} } $. 
	Hence, $V^{\pi}(s)$, $s \in \tilde{ {\cal S} }$ can be computed by solving the following set of equations: 
	\begin{align}
		V^{\pi}(G) = 1, ~
		V^{\pi}(D) = 0, ~ V^{\pi}(s) = Q^{\pi}(s, \pi(s) ), ~ \forall s \in \tilde{ {\cal S} } \nonumber
	\end{align}
	where $\pi(s)$ denotes the action selected by $\pi$ in state $s$.
	The value of policy $\pi$ is defined as 
	\begin{align}
		V^{\pi} :=  \sum_{s \in \tilde{ {\cal S} }}  q(s) V^{\pi}(s) . \notag
	\end{align}
	
	The optimal policy is denoted by \begin{equation}
		\pi^* := \argmax_{\pi \in \Pi} V^{\pi} \nonumber
	\end{equation}
	and the value of the optimal policy is denoted by 
	\begin{equation}
		V^{*} := \max_{\pi \in \Pi} V^{\pi}. \nonumber
	\end{equation}
	The optimal policy is characterized by Bellman optimality equations for all $s \in \tilde{S}$
	\begin{align} \label{eqn:bellval}
		V^{*}(s) &= \max \{ p^F V^*(G) , p^C V^{*}(s+1) +  p^D V^{*}(s-1)     \}, \notag \\
		& = \max \{ p^F , p^C V^{*}(s+1) +  p^D V^{*}(s-1) \}.
	\end{align}
	As it is sufficient to search for the optimal policy within stationary deterministic Markov policies and since there are only two actions that can be taken in each $s \in \tilde{ {\cal S} }$, the number of all such policies is $2^{G-1}$.
	In Section \ref{sec:structure}, we will prove that the optimal policy for GRBP has a simple threshold form, which reduces the number of policies to learn from $2^{G-1}$ to $2$.
	\subsection{Online Learning in the GRBP}
	As we described in the previous subsection, when the state transition probabilities are known, optimal solution and its probability of reaching to the goal can be found by solving Bellman optimality equations. 
	When the learner does not know $p^C$ and $p^{F}$, the optimal policy cannot be computed a priori, and hence needs to be learned. 
	We define the learning loss of the learner, who is not aware of the optimal policy a priori, with respect to an oracle, who knows the optimal policy from the initial round, as the regret given by 
	\begin{align}
		\text{Reg}(T) := TV^* - \sum_{ \rho= 1 }^{ T } V^{ \hat{\pi}_\rho } \nonumber
	\end{align} 
	where $\hat{\pi}_\rho$ denotes the policy that is used by the learner in round $\rho$. Let $N_{\pi}(T)$ denote the number of times policy $\pi$ is used by the learner by round $T$. 
	For any policy $\pi$, let $\Delta_{\pi} := V^* - V^{\pi}$ denote the suboptimality gap of that policy. The regret can be rewritten as 
	\begin{align}
		\text{Reg}(T) &= \sum_{\pi \in \Pi}  N_{\pi}(T) \Delta_{\pi} .  \label{eqn:RegretDefinf}
	\end{align}
	
	In this paper, we will design learning algorithms that minimize the growth rate of the expected regret, i.e., $\mathbb{E} [ \text{Reg}(T) ]$. 
	A straightforward way to do this will be to employ UCB1 algorithm \cite{auer2002gdmdp} or its variants \cite{auer2010ucb} by taking each policy as an arm. The result below state a logarithmic bound on the expected regret when UCB1 is used.
	\begin{theorem} \label{thm:UCB1Reg}
		When UCB1 in \cite{auer2002gdmdp} is used to select the policy to follow at the beginning of each round (with set of arms $\Pi$), we have 
		\begin{equation}
			\mathbb{E} [ \text{Reg}(T) ] = 8 \sum_{\pi : V^{\pi} < V^*} 
			\dfrac{\log T}{\Delta_{\pi}} + \left (1+\dfrac{\pi^2}{3} \right) \sum_{\pi \in \Pi} \Delta_{\pi}. \nonumber
		\end{equation}
	\end{theorem}
	\begin{proof}
		See \cite{auer2002gdmdp}.
	\end{proof}
	As shown in Theorem \ref{thm:UCB1Reg}, the expected regret of UCB1 depends linearly on the number of suboptimal policies. For GRBP, the number of policies can be very large.
	For instance, we have $2^{G-1}$ different stationary deterministic Markov policies for the defined problem. These imply that using UCB1 to learn the optimal policy is highly inefficient for the GRBP.
	The learning algorithm we propose in Section \ref{sec:algorithm} exploits a result on the form of the optimal policy that will be derived in Section \ref{sec:structure} to learn the optimal policy in a fast manner. This learning algorithm calculates an estimated optimal policy using the estimated transition probabilities, and hence learns much faster than applying UCB1 naively. Moreover, it can even achieve bounded regret (instead of logarithmic regret) under some special cases.
	% In contrast to UCB1, the number of times action $C$ is selected can be greater than one during a single round by using the proposed learning algorithm. In addition, the learning algorithm we will propose can achieve finite regret for one case of GRBP.
	\section{Form of the Optimal Policy} \label{sec:structure}
	In this section, we prove that the optimal policy for GRBP has a threshold form. The value of the threshold depends only on the state transition probabilities and the number of states.
	%
	%\subsection{Structure of the Optimal Policy for GRBP}
	First, we give the definition of a {\em stationary threshold} policy. 
	\begin{definition}
		$\pi$ is a stationary threshold policy if there exists $\tau \in \{0, 1, \ldots, G-1\}$ such that
		$\pi(s) = C$ for all $s > \tau$ and $\pi(s) = F$ for all $s \leq \tau$. We use $\pi^{tr}_{\tau}$ to denote the stationary threshold policy with threshold $\tau$. The set of stationary threshold policies is given by $ \Pi^{tr} := \{\pi^{tr}_\tau\}_{\tau = \{0, 1, \ldots, G-1\}}$.
	\end{definition}
	The next lemma constrains the set of policies that the optimal policy lies in.
	\begin{lemma} \label{lemma:GINFreduction}
		In the GRBP it is always optimal to select action $C$ at
		$s \in \tilde{\cal{S}}-\{1\}$.
	\end{lemma}
	\begin{proof}
		By \eqref{eqn:bellval}, for $s \in \tilde{\cal S}-\{1\}$ we have
		\begin{align}
			V^{\ast}( s ) = \max \{p^F, p^C V^{\ast}( s+1 ) + p^D V^{\ast}( s - 1 ) \}. \nonumber
		\end{align}
		If $V^{\ast}( s ) = p^F$, this implies that
		\begin{align} \label{eqn:maxfunc}
			& p^C V^{\ast}( s+1 ) + p^D V^{\ast}( s - 1 ) \leq p^F \Rightarrow \nonumber \\
			& V^{\ast}( s-1 ) \leq \frac{ p^F - p^C V^{\ast}(s+1) } { p^D } .
		\end{align}
		By definition,
		\begin{align}
			p^F \leq V^{\ast}(s), \forall s\in \tilde{ {\cal S} } . \label{eqn:Vlower}
		\end{align}
		Therefore,
		\begin{align}
			\frac{ p^F - p^CV^{\ast}( s+1 ) }{ p^D } 
			& \leq \frac{ p^F - p^C p^F }{ p^D } = p^F \nonumber
		\end{align}
		which in combination with \eqref{eqn:maxfunc} implies that $V^{\ast}( s-1 ) \leq p^F$. 
		According to  
		\eqref{eqn:Vlower} we find that $V^{\ast}( s-1 ) = p^F$.
		Then, we conclude that
		\begin{align}
			V^{\ast}( s ) = p^F \Rightarrow V^{\ast}( s-1 ) = p^F , \forall s \in \tilde{\cal S}-\{1\} .  \notag
		\end{align} 
		This also implies that
		\begin{equation}
			V^{\ast}( s+1 ) \leq \frac{p^F - p^DV^{\ast}( s-1 )}{p^C} 
			= p^F . \nonumber
		\end{equation}
		Consequently, if $V^{\ast}( s ) = p^F \text{ for some } s \in \tilde{\cal S}-\{1\}$, then
		\begin{align}
			V^{\ast}( s ) = p^F,  \forall s \in \tilde{\cal S}-\{1\}  .  \label{eqn:contradict}
		\end{align}
		By \eqref{eqn:contradict}, if $V^{\ast}( s ) = p^F$ for some 
		$s \in \tilde{\cal S}-\{1\}$, then this implies that $V^{\ast}( G-1 ) = p^F$. Since $V^{\ast}( G ) = 1$, we have
		\begin{align}
			V^*(G-1) & = \max \{ p^F, p^C + p^D p^F \} = p^F \notag \\
			& \Rightarrow p^F \geq p^C + p^D p^F \notag \\
			& \Rightarrow p^F (1 - p^D  ) \geq p^C \Rightarrow p^F \geq 1 \Rightarrow p^F = 1 . \notag
		\end{align}
		This shows that unless $p^F = 1$, it is suboptimal to select action $F$ in states $\tilde{\cal S}-\{1\}$ and since $p^F = 1$ is a trivial case, we disregard that. Hence, it is always optimal to select action $C$ at $s \in \tilde{\cal{S}}-\{1\}$.
	\end{proof}
	The result of Lemma \ref{lemma:GINFreduction} holds independently from the set of transition probabilities and the number of states. 
	Lemma \ref{lemma:GINFreduction} leaves out only two candidates for the optimal policy. 
	The first candidate is the policy which selects action $C$ at any state $s \in \tilde{\cal{S}}$. The second candidate selects action $C$ in all states except state $1$.
	Hence, the optimal policy is always in set $\{\pi^{tr}_0, \pi^{tr}_1\}$.
	This reduces the set of policies to consider from $2^{G-1}$ to $2$.
	Let $r := p^{D}/p^{C}$ denote the {\em failure ratio} of action $C$. The next lemma gives the value functions for $\pi^{tr}_1$ and $\pi^{tr}_0$.
	\begin{lemma} \label{lemma:ExpRew}
		In the GRBP we have \\
		(i) $ V^{\pi^{tr}_1}(s) =
		\begin{cases}
		p^F + (1-p^F)\dfrac{1-r^{s-1}}{1-r^{G-1}}, ~ 
		& \text{when} ~ r \neq 1 \notag \\
		p^F + (1-p^F)\dfrac{s-1}{G-1}, ~ 
		& \text{when} ~ r = 1 \notag	 
		\end{cases} $ \\
		(ii) $V^{\pi^{tr}_0}(s) =
		\begin{cases}
		\dfrac{1-r^{s}}{1-r^{G}}, ~ & \text{when} ~~ r \neq 1 \notag\\
		\dfrac{s}{G}, ~ & \text{when} ~~ r = 1 \notag
		\end{cases}$ \\
		for $s \in \tilde{\cal S}$. 
	\end{lemma}
	\begin{proof}
		\textbf{(i)}: 
		
		%		According to \eqref{eqn:bellval} and \eqref{eqn:Qfnc}, this case occurs if $p^F \geq p^C V^*(2)$.
		For $\pi^{tr}_1$ we have:
		\begin{align}
			& \begin{cases}
				V^{\pi^{tr}_1}(G) = 1 \\
				V^{\pi^{tr}_1}(G-1) = p^C V^{\pi^{tr}_1}(G) + p^D V^{\pi^{tr}_1}(G-2) \\
				\dots\\
				V^{\pi^{tr}_1}(2) = p^C V^{\pi^{tr}_1}(3) + p^D V^{\pi^{tr}_1}(1)\\
				V^{\pi^{tr}_1}(1)=p^F\\
			\end{cases}\nonumber\\
			\Rightarrow 
			& \begin{cases}
				(p^C+p^D)V^{\pi^{tr}_1}(G-1) = p^C+p^DV^{\pi^{tr}_1}(G-2)\\
				\dots\\
				(p^C+p^D)V^{\pi^{tr}_1}(2) = p^CV^{\pi^{tr}_1}(3)+p^Dp^F\\
			\end{cases} \nonumber\\
			\Rightarrow 
			& \begin{cases}
				p^C(V^{\pi^{tr}_1}(G-1)-1)= \\ p^D(V^{\pi^{tr}_1}(G-2) - V^{\pi^{tr}_1}(G-1))\\
				\dots\\
				p^C(V^{\pi^{tr}_1}(s+1)-V^{\pi^{tr}_1}(s+2)) 
				= \\ p^D(V^{\pi^{tr}_1}(s)-V^{\pi^{tr}_1}(s+1))\\
				\dots\\
				p^C(V^{\pi^{tr}_1}(2)-V^{\pi^{tr}_1}(3)) = p^D(p^F-V^{\pi^{tr}_1}(2))\\
			\end{cases} \nonumber\\
			\Rightarrow 
			& \begin{cases}
				V^{\pi^{tr}_1}(G-1)-1 = r^{G-2}(p^F-V^{\pi^{tr}_1}(2))\\
				\dots\\
				V^{\pi^{tr}_1}(s)-V^{\pi^{tr}_1}(s+1) = r^{s-1}(p^F-V^{\pi^{tr}_1}(2)) \\
				\dots\\
				V^{\pi^{tr}_1}(2)-V^{\pi^{tr}_1}(3) = r(p^F-V^{\pi^{tr}_1}(2))\\
			\end{cases}\Rightarrow \label{eqn:V2}
		\end{align}
		Summation of all the terms results in
		\begin{align}
			& 1-V^{\pi^{tr}_1}(2) = (V^{\pi^{tr}_1}(2)-p^F) (\sum_{i=1}^{G-2} r^i)\label{eqn:rone} \Rightarrow\\
			& V^{\pi^{tr}_1}(2)(\sum_{i=0}^{G-2} r^i) = 1 + p^F(\sum_{i=1}^{G-2} r^i) \Rightarrow \nonumber \\
			& V^{\pi^{tr}_1}(2)(\sum_{i=0}^{G-2} r^i) = 1 - p^F + p^F(\sum_{i=0}^{G-2} r^i) \Rightarrow \nonumber \\
			& V^{\pi^{tr}_1}(2) = p^F + \frac{1 - p^F}{(\sum_{i=0}^{G-2} r^i)} \Rightarrow \nonumber \\
			& V^{\pi^{tr}_1}(2) = p^F+(1-p^F)\dfrac{1-r}{1-r^{G-1}}. \notag
		\end{align}
		Then, for $s$th state, we have to sum up to $(s-1)$th equation in \eqref{eqn:V2}:
		\begin{align}
			&V^{\pi^{tr}_1}(s)-V^{\pi^{tr}_1}(2) = (V^{\pi^{tr}_1}(2)-p^F) (\sum_{i=1}^{ s-2 } r^i) \nonumber \Rightarrow \\ 
			& V^{\pi^{tr}_1}(s) = p^F + (V^{\pi^{tr}_1}(2) - p^F) (\sum_{i=0}^{ s-2 } r^i) \Rightarrow \label{eqn:rtwo} \\ 
			& V^{\pi^{tr}_1}(s) = p^F+(1-p^F)\dfrac{1-r^{s-1}}{1-r^{G-1}} . \label{eqn:value1}
		\end{align}
		For the fair case, $r$ has to be set to 1 in \eqref{eqn:rone} and \eqref{eqn:rtwo}. Then,
		\begin{align}
			V^{\pi^{tr}_1}(2) = p^F+(1-p^F)\dfrac{1}{G-1}     \notag
		\end{align}
		and
		\begin{align}
			V^{\pi^{tr}_1}(s) = p^F+(1-p^F)\dfrac{s-1}{G-1}  .    \notag
		\end{align}
		\\
		\textbf{Case (ii)}: 
		
		%		According to \eqref{eqn:bellval} and \eqref{eqn:Qfnc}, this case occurs if $p^F < p^C V^*(2)$. 
		Since action $F$ is never selected by $\pi^{tr}_0$,
		for this case, standard analysis of the gambler's ruin problem applies. Thus, the probability of hitting $G$ from state $s$ is 
		\begin{align}
			(1-r^s)/(1-r^G)     \label{eqn:value2} 
		\end{align}
		for $r \neq 1$ and $s/G$ for $r = 1$ \cite{el-shehawey2009grfmc}.
	\end{proof}
	The form of the optimal policy is given in the following theorem.	
	\begin{theorem} \label{thm:tauinf}
		In the GRBP, the optimal policy is $\pi^{tr}_{\tau^*}$, where \\
		\begin{equation}
			\tau^* = 
			\begin{cases}
			\sgn(p^{F}- \dfrac{1-r}{1-r^{G}}), ~ & \text{ when $r\neq 1$}\notag\\
			\sgn(p^{F}-\dfrac{1}{G}), ~ & \text{ when $r = 1$}\notag
			\end{cases}
		\end{equation}
		where $\sgn(x) = 1$ if $x$ is nonnegative and $0$ otherwise.
	\end{theorem}
	\begin{proof}	
		Since we have found in Lemma \ref{lemma:GINFreduction} that it is always optimal to select action $C$ when the state is in $\{2, \ldots, G-1 \}$, to find the optimal policy, it is sufficient to compare the value functions of the two policies for $s=1$. When $r \neq 1$, this gives  
		$\pi^* = \pi^{tr}_1$ if
		\begin{align}
			\dfrac{1-r}{1-r^G}  < p^F      \notag
		\end{align}
		and $\pi^* = \pi^{tr}_0$ otherwise.\footnote{When $(1-r)/(1-r^G)  = p^F$ both $\pi^{tr}_1$ and $\pi^{tr}_0$ are optimal. For this case, we favor $\pi^{tr}_1$ because it always ends the current round.} 
		Similarly, if $r = 1$ and $1/G < p^F$, then $\pi^* = \pi^{tr}_1$. Otherwise, $\pi^* = \pi^{tr}_0$. Using these, the value of the optimal threshold is given as  
		\begin{align}
			\tau^* = 
			\begin{cases}
				\sgn(p^F- \dfrac{1-r}{1-r^{G}}) & \text{if } r \neq 1 \notag\\
				\sgn(p^F-\dfrac{1}{G}) & \text{if } r = 1 \notag
			\end{cases}
		\end{align}
		which completes the proof.
	\end{proof}
	When $r \neq 1$, the term $(1-r) / (1-r^G)$ represents probability of hitting $G$ starting from state $1$ by always selecting action $C$. This probability is equal to $1/G$ when $r = 1$.
	Because of this, it is optimal to take the terminal action in some cases for which $p^C > p^{F}$. Although the continuation action can move the system state in the direction of the goal state for some time, the long term chance of hitting the goal state by taking the continuation action can be lower than the chance of hitting the goal state by immediately taking the terminal action at state $1$. 
	
	Equation of the boundary for which the optimal policy changes from $\pi^{tr}_0$ to $\pi^{tr}_1$ is 
	\begin{equation} \label{eqn:boundaryP}
		p^{F} = B(r) := (1-r) / (1-r^{G}) 
	\end{equation}
	when $r \neq 1$.
	This decision boundary is illustrated in Fig. \ref{fig:Boundary} for different values of $G$. We call the region of transition probabilities for which $\pi^{tr}_0$ is optimal as the {\em exploration} region, and the region for which $\pi^{tr}_1$ is optimal as the {\em no-exploration} region. In exploration region, the optimal policy does not take action $F$ in any round. Therefore, any learning algorithm that needs to learn how well action $F$ performs, needs to explore action $F$.
	As the value of $G$ increases, area of the exploration region decreases due to the fact that probability of hitting the goal state by only taking action $C$ decreases.
	\begin{figure}[h]
		\begin{center}
			\includegraphics[scale=0.27]{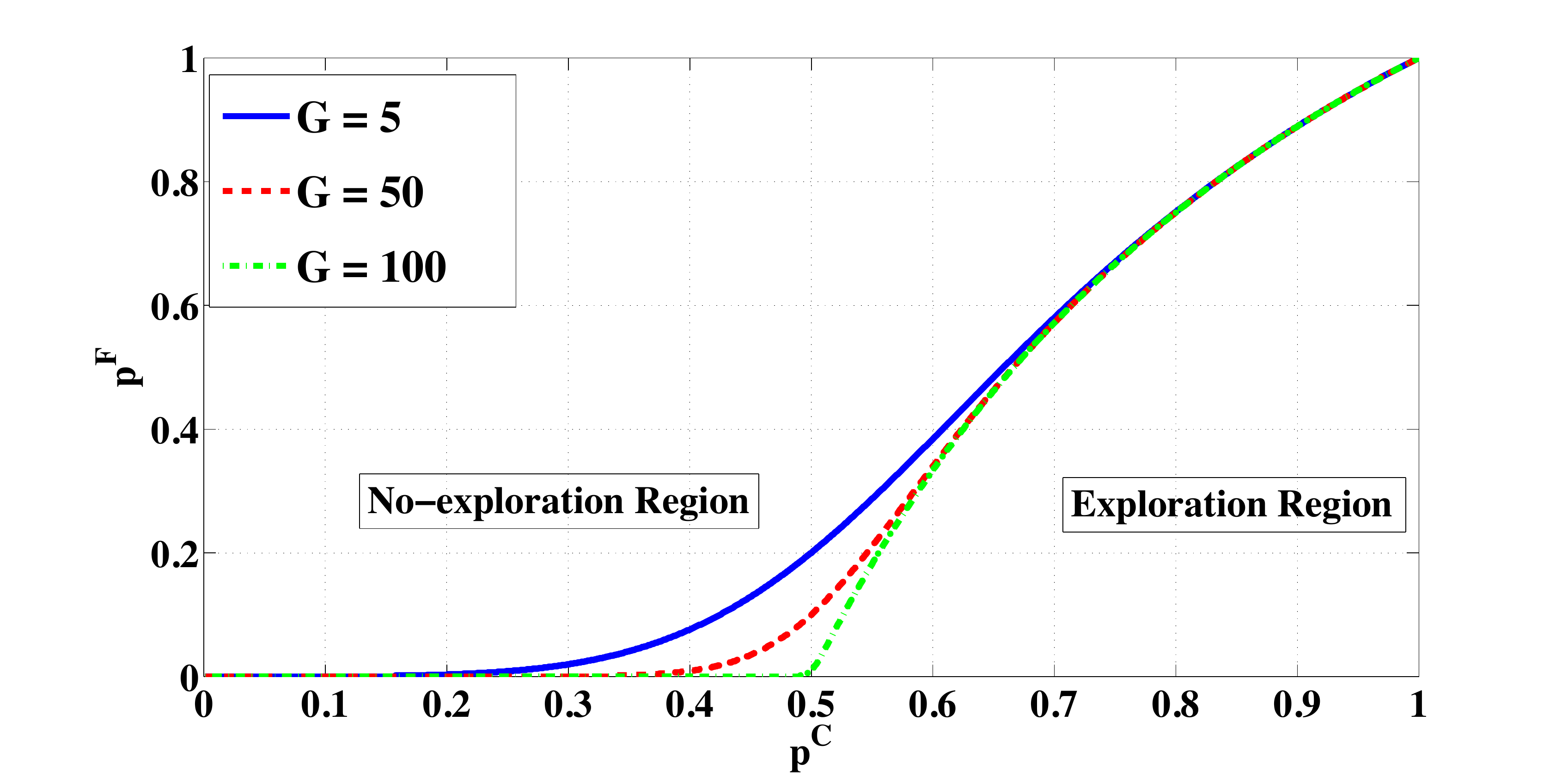}
			\caption{The boundary between explore and no-explore regions} \label{fig:Boundary}
		\end{center}
	\end{figure}
	
	%In Section \ref{sec:algorithm}, a learning algorithm that learns the optimal policy by repeated interaction is given. This algorithm applies greedy maximization in learning the optimal policy if the estimated optimal policy is $\pi^{tr}_1$, while it uses the separation of exploration and exploitation principle with control function $D(\rho)$ if the estimated optimal policy is $\pi^{tr}_0$ (exploration region) to ensure that action $F$ is taken sufficiently many times to have an accurate estimate about its transition probability. We will show that this algorithm achieves finite regret when the actual transition probabilities lie in the no-exploration region, and will achieve logarithmic regret when the actual transition probabilities lie in the exploration region.
	
	\section{An Online Learning Algorithm and Its Regret Analysis} \label{sec:algorithm}
	In this section, we propose a learning algorithm that minimizes the regret when the state transition probabilities are unknown. 
	The proposed algorithm forms estimates of state transition probabilities based on the history of state transitions, and then, uses these estimates together with the form of the optimal policy obtained in Section \ref{sec:structure} to calculate an estimated optimal policy at each round. 
	
	\subsection{Greedy Exploitation with Threshold Based Exploration}
	The learning algorithm for the GRBP is called {\em Greedy Exploitation with Threshold Based Exploration} (GETBE) and its pseudocode is given in Algorithm \ref{alg:GRBP}.
	Unlike conventional MAB algorithms \cite{lai1985aeaar,auer2002gdmdp,auer2010ucb} which require all arms to be sampled at least logarithmically many times, GETBE does not need to sample all policies (arms) logarithmically many times to find the optimal policy with a sufficiently high probability. 
	GETBE achieves this by utilizing the form of the optimal policy derived in the previous section.
	Although GETBE does not require all policies to be explored, it requires exploration of action $F$ when the estimated optimal policy never selects action $F$. 
	This {\em forced} exploration is done to guarantee that GETBE does not get stuck in the suboptimal policy.
	
	GETBE keeps counters $N^G_F(\rho)$, $N_F(\rho)$, $N^{u}_C(\rho)$ and $N_C(\rho)$:
	(i) $N^G_F(\rho)$ is the number of times action $F$ is selected and terminal state $G$ is entered upon selection of action $F$ by the beginning of round $\rho$, (ii) $N_F(\rho)$ is the number of times action $F$ is selected by the beginning of round $\rho$,
	(iii) $N^{u}_C(\rho)$ is the number of times transition from some state $s$ to $s+1$ happened (i.e., the state moved up) after selecting action $C$ by the beginning of round $\rho$, (iv) $N_C(\rho)$ is the number of times action $C$ is selected by the beginning of round $\rho$.
	Let $T_F(\rho)$ and $T_C(\rho)$ represent the number of times action $F$ and action $C$ is selected in round $\rho$, respectively. 
	Since, action $F$ is a terminal action, it can be selected at most once in each round. However, action $C$ can be selected multiple times in the same round. 
	Let $T^G_F(\rho)$ and $T^{u}_C(\rho)$ represent the number of times state $G$ is reached after the selection of action $F$ and the number of times the state moved up after the selection of action $C$ in round $\rho$, respectively. 
	
	At the beginning of round $\rho$, GETBE forms the transition probability estimates $\hat{p}^{F}_\rho := N^G_F(\rho) / N_F(\rho)$ and 
	$\hat{p}^C_\rho := N^{u}_C(\rho) / N_C(\rho)$ that correspond to actions $F$ and $C$, respectively. 
	Then, it computes the estimated optimal policy $\hat{\pi}_{\rho}$ by using the form of the optimal policy given in Theorem 2 for the GRBP.
	If $\hat{\pi}_{\rho} = \pi^{tr}_1$, then GETBE operates in {\em greedy exploitation mode} by acting according to $\pi^{tr}_1$ for the entire round. 
	Else if $\hat{\pi}_{\rho} = \pi^{tr}_0$, then GETBE operates in {\em triggered exploration mode} and selects action $F$ in the first time slot of that round if $N_F(\rho) < D(\rho)$, where $D(\rho)$ is a non-decreasing {\em control function} that is an input of GETBE.
	This control function helps GETBE to avoid getting stuck in the suboptimal policy by forcing the selection of action $F$, although it is suboptimal according to $\hat{\pi}_{\rho}$. When $N_F(\rho) \geq D(\rho)$, GETBE employs $\hat{\pi}_{\rho}$ for the entire round. 

	At the end of round $\rho$ the values of counters are updated as follows:
	\begin{align}
		N_F(\rho+1) &= N_F(\rho) + T_F(\rho) \notag \\
		N^G_F(\rho+1) &= N^G_F(\rho) + T^G_F(\rho) \notag \\ 
		N_C(\rho+1) &= N_C(\rho) + T_C(\rho) \notag \\
		N^{u}_C(\rho+1) &= N^{u}_C(\rho) + T^{u}_C(\rho). \label{eqn:updcounts}
	\end{align}
	These values are used to estimate the transition probabilities that will be used at the beginning of round $\rho + 1$, for which the above procedure repeats. 
	In the analysis of GETBE, we will show that when $N_F(\rho) \geq D(\rho)$, the probability that GETBE selects the suboptimal policy is very small, which implies that the regret incurred is very small.
	
	\begin{algorithm}[h]
		\algsetup{linenosize= \small}
		\small
		\caption{GETBE Algorithm} \label{alg:GRBP}
		\begin{algorithmic}[1]
			\STATE $\textit{Input}: G, D(\rho)$
			\STATE $\textit{Initialize}$: Take action $C$ and then action $F$ once to form initial estimates: $N^{G}_{F}(1), N_{F}(1) = 1, N^{u}_{C}(1), N_{C}(1) = 1$ (Round(s) to form the initial estimates (at most 2 rounds) are ignored in the regret analysis). $\rho = 1$ \label{line:Init}
			\WHILE{$\rho\geq 1$}
			\STATE Get initial state $s^\rho_1 \in \tilde{ {\cal S} }$, $t=1$
			\STATE $\hat{p}^{F}_{\rho} = \dfrac{N^G_{F}(\rho)}{N_{F}(\rho)}, ~ \hat{p}^C_{\rho} = \dfrac{N^{u}_{C}(\rho)}{N_{C}(\rho)}, ~ \hat{r}_{\rho} = \dfrac{1-\hat{p}^C_{\rho}}{\hat{p}^C_{\rho}}$
			\IF {$\hat{p}^{u}_{\rho} = 0.5$}
			\STATE $\hat{\tau}_{\rho} = \text{ sign}(\hat{p}^{F}_{\rho}-1/G)$
			\ELSE
			\STATE $\hat{\tau}_{\rho} = \text{ sign}(\hat{p}^{F}_{\rho}-\dfrac{1-\hat{r}_{\rho}}{1-(\hat{r}_{\rho})^G})$
			\ENDIF
			\STATE Set $\hat{\pi}_\rho = \pi^{tr}_{\hat{\tau}_\rho}$
			\WHILE {$s^{\rho}_t \neq$ $G$ or $D$}
			\IF {$ (\hat{\pi}_\rho = \pi^{tr}_0  ~~ \&\& ~~ N_F(\rho) < D(\rho) ) ~~ || ~~  (s^{\rho}_t \leq \hat{\tau}_{\rho} ) $}
			\STATE Select action $F$, observe state $s^\rho_{t+1}$
			\STATE $T_{F}(\rho) =T_{F}(\rho) +1, ~ T^{G}_F(\rho)  = \mathbb{I} (  s^\rho_{t+1} = G ) $\footnotemark
			\ELSE
			\STATE Select action $C$, observe state $s^\rho_{t+1}$
			\STATE $T_{C}(\rho)= T_{C}(\rho)+1$
			\STATE $T^{u}_{C}(\rho) = T^{u}_{C}(\rho)+ \mathbb{I} ( s^{\rho}_{t+1} = s^{\rho}_t + 1 )$
			\STATE $t=t+1$
			\ENDIF
			\ENDWHILE
			\STATE $\text{Update the counters according to \eqref{eqn:updcounts}}$
			\STATE $\rho=\rho+1$
			\ENDWHILE	
		\end{algorithmic}
	\end{algorithm}
	\footnotetext{$\mathbb{I}(\cdot)$ denotes the indicator function which is $1$ if the expression inside evaluates true and $0$ otherwise.}
	
	\subsection{Regret Analysis}
	In this section, we bound the (expected) regret of GETBE. We show that GETBE achieves bounded regret when the unknown transition probabilities lie in no-exploration region and logarithmic (in number of rounds) regret when the unknown transition probabilities lie in exploration region.
	Based on Theorem \ref{thm:tauinf}, GETBE only needs to learn the optimal policy from the set of policies $\{ \pi^{tr}_0, \pi^{tr}_1 \}$. 
	Using this fact and taking the expectation of \eqref{eqn:RegretDefinf}, the expected regret of GETBE can be written as
	\begin{equation}
		\mathbb{E} [ \text{Reg}(T) ] = 
		\sum_{\pi \in \{  \pi^{tr}_0, \pi^{tr}_1 \}}  
		\mathbb{E} [ N_{\pi}(T) ] \Delta_{\pi}. \label{eqn:regretreduced}
	\end{equation}
	Let $ \Delta(s) :=  | V^{\pi^{tr}_1}(s) -  V^{\pi^{tr}_0}(s) |, s \in \tilde{{\cal S}} $ be the suboptimality gap when the initial state is $s$.
	For any $\pi \in \{ \pi^{tr}_0, \pi^{tr}_1 \}$, we have $ \Delta_{\pi} \leq \Delta_{\max}$, where $\Delta_\max := \max_{ s \in \tilde{{\cal S}} } \Delta(s) $.  
	The next lemma gives closed-form expressions for $\Delta(s)$ and $\Delta_{\max}$.
	\begin{lemma} \label{lemma:DeltaBound}
		We have
		\begin{align}
			& \Delta(s) =
			\begin{cases}
				\frac{G-s}{G-1} |p^F - \dfrac{1}{G}| & \text{if } r = 1 \notag \\ \frac{r^{G-1}-r^{s-1}}{r^{G-1}-1} |p^F - \dfrac{1-r}{1-r^G}| & \text{if } r \neq 1 \notag
			\end{cases} \nonumber \\
			& \text{ and } \nonumber \\
			& \Delta_{\max} =
			\begin{cases}
				|p^F - \dfrac{1}{G}| & \text{if } r = 1 \notag \\
				|p^F - \dfrac{1-r}{1-r^G}| & \text{if } r \neq 1 \notag
			\end{cases}
		\end{align}
	\end{lemma}
	\begin{proof}
		According to Lemma \ref{lemma:ExpRew} we have \\
		\textbf{Case (i)} $r = 1$:
		\begin{align}
			\Delta(s) & = |V^{\pi^{tr}_1}(s) - V^{\pi^{tr}_0}(s)| = |p^F + (1-p^F)\frac{s-1}{G-1} - \frac{s}{G}| \nonumber \\ 
			& = |p^F(\dfrac{G-s}{G-1}) + \frac{s-1}{G-1} - \frac{s}{G}| = \dfrac{G-s}{G-1}|p^F - \frac{1}{G}| . \notag
		\end{align}
		The above equation is maximized when $s = 1$. Therefore, when $r = 1$,
		\begin{equation}
			\Delta_{\max} = \max_{ s \in \tilde{ {\cal S} } } \Delta(s) = |p^F - \frac{1}{G}| . \nonumber
		\end{equation}	
		\textbf{Case (ii)} $r \neq 1$:
		\begin{align} 
			\Delta(s) & = |V^{\pi^{tr}_1}(s) - 	V^{\pi^{tr}_0}(s)| \nonumber \\ 
			& = |p^F + (1-p^F)\dfrac{r^{s-1}-1}{r^{G-1}-1} - \frac{r^s-1}{r^G-1}| \nonumber \\ 
			& = |p^F(\dfrac{r^{G-1}-r^{s-1}}{r^{G-1}-1}) + \frac{r^{s-1}-1}{r^{G-1}-1} - \frac{r^s-1}{r^G-1}| \nonumber \\
			& = |p^F(\frac{r^{G-1}-r^{s-1}}{r^{G-1}-1}) + \frac{r^{s}-r^{s-1}+r^{G-1}-r^G}{(1-r^{G-1})(1-r^G)}| \nonumber \\
			& = (\frac{r^{G-1}-r^{s-1}}{r^{G-1}-1})|p^F - \frac{1-r}{1-r^G}| . \notag
		\end{align}
		Again, the above equation is maximized when $s = 1$. Therefore, when $r \neq 1$,
		\begin{equation}
			\Delta_{\max} = \max_{ s \in \tilde{ {\cal S} } } \Delta(s) = |p^F - \frac{1-r}{1-r^G}| . \nonumber
		\end{equation}
	\end{proof}
	Next, we bound $\mathbb{E} [ N_{\pi}(T) ]$ for the suboptimal policy in a series of lemmas. 
	From \eqref{eqn:boundaryP}, it is clear that the boundary is a function of $r$. Let $r = \frac{1-x}{x}$. Then, the boundary becomes a function of $x$ by which we have 
	\begin{equation}
		B(x) = ( 1-\dfrac{1-x}{x} ) /
		(1- ( \dfrac{1-x}{x} )^{G} ) \nonumber.
	\end{equation}
	Let $\delta$ be the minimum Euclidean distance of pair ($p^C, p^F$) from the boundary ($x, B(x)$) given in Fig. \ref{fig:Boundary}. 
	The value of $\delta$ specifies the {\em hardness} of GRBP. When $\delta$ is small, it is harder to distinguish the optimal policy from the suboptimal policy. If the pair of estimated transition probabilities $(\hat{p}^{C}_{\rho}, \hat{p}^{F}_{\rho}$) in round $\rho$ lies within a ball around ($p^C, p^F$) with radius less than $\delta$, then GETBE will select the optimal policy in that round. 
	The probability that GETBE selects the optimal policy is lower bounded by the probability that the estimated transition probabilities lie in a ball centered at ($p^C, p^F$) with radius $\delta$.
	
	The following lemma provides a lower bound on the expected number of times each action is selected by GETBE. This result will be used when bounding the regret of GETBE.  
	\begin{lemma} \label{lemma:MinNumPBinf}
		\textbf{(i)} Let $p_{F, 1}$ be the probability of taking action $F$ in round $\rho$ when $\hat{\pi}_{\rho} = \pi^{tr}_1$ and $p_{C, 1}$ be the probability of taking action $C$ at least once in round $\rho$ when $\hat{\pi}_{\rho} = \pi^{tr}_1$. Then,
		\begin{align}
			p_{C, 1} & = 1 - q(1) \nonumber \\
			p_{F, 1} & = 
			\begin{cases}
			\sum_{s=1}^{G-1} { \frac{G-s}{G-1} q(s) } & \text{if } r=1 \\
			\sum_{s=1}^{G-1} { \frac{r^{s-1} - r^{G-1}}{1-r^{G-1}} q(s) } & \text{if } r\neq 1 \nonumber
			\end{cases}.
		\end{align}
		\textbf{(ii)} Let $D(\rho) := \gamma\log\rho$ where $\gamma > 1/p^2_{F, 1}$, and
		\begin{equation}
			f_a(\rho) := 
			\begin{cases}
			0.5 p_{C, 1} \rho, &  \text{for } a =C \nonumber \\
			0.5(p_{F, 1}\gamma - \sqrt{\gamma}) \log \rho, &  \text{for } a = F \nonumber
			\end{cases}
		\end{equation}
		Let $\rho^\prime_C$ be the first round in which $0.5p_{C, 1}\rho - p_{C, 1}\lceil D(\rho)\rceil - \sqrt{\rho - \lceil D(\rho)\rceil \log \rho}$ becomes positive and $\rho^\prime_F$ be the first round in which both $0.5 (p_{F, 1}\gamma - \sqrt{\gamma}) \log \rho - \sqrt{\log\rho}$ and $\rho - 2 - \lceil D(\rho) \rceil$ becomes positive. Then for $a \in \{F, C\}$ we have
		\begin{align}
			& \Pr\left( N_{a}(\rho) < f_a(\rho) \right) \leq \dfrac{1}{\rho^2}, \text{ for } \rho \geq \rho^\prime_a. \nonumber
		\end{align}
	\end{lemma}
	\begin{proof}
		The following expressions will be used in the proof:
		\begin{itemize}
			\item $N_0(\rho):$ Number of rounds by $\rho$ for which $\hat{\pi}_{\rho} = \pi^{tr}_0$.
			\item $N_1(\rho):$ Number of rounds by $\rho$ for which $\hat{\pi}_{\rho} = \pi^{tr}_1$.
			\item $N_{F, 1}(\rho):$ Number of rounds by $\rho$ for which action $F$ is taken when $\hat{\pi}_{\rho} = \pi^{tr}_1$.
			\item $N_{C, 1}(\rho):$ Number of rounds by $\rho$ for which action $C$ is taken when $\hat{\pi}_{\rho} = \pi^{tr}_1$.
			\item $n_a(\rho):$ Indicator function of the event that action $a$ is selected for at least once in round $\rho$.
		\end{itemize}
		\textbf{(i)} When $\hat{\pi}_{\rho} = \pi^{tr}_1$, action $C$ is not taken only if the initial state is $1$. Hence, 
		\begin{align}
			p_{C, 1} & = 1 - \Pr( s^\rho_1 = 1)	= 1 - q(1). \notag
		\end{align} 
		Let $H_1$ denote the event that state $1$ is reached before state $G$ when $\hat{\pi}_{\rho} = \pi^{tr}_1$. We have
		\begin{align} 
	%		\label{eqn:pf1inf}
			p_{F, 1} = 
			\sum_{s=1}^{G-1} { \Pr( H_1 | s^{\rho}_1 = s)  q(s) } . \notag
		\end{align}
		When $r = 1$, $p_{F, 1}$ is equivalent to the {\em ruin} probability (probability of hitting the terminal state $1$) of a fair gambler's ruin problem over $G-1$ states, where states 1 and $G$ are the terminal states. For this problem, the probability of hitting $G$ from state $s$ is $\frac{s-1}{G-1}$. Hence, probability of hitting state $1$ from state $s$ is
		\begin{align}
			\Pr( H_1 | s^{\rho}_1 = s)  = 1 - \frac{s-1}{G-1} = \frac{G-s}{G-1} . \nonumber
		\end{align}
		When $r \neq 1$, the problem is equivalent to an unfair gambler's ruin problem with $G-1$ states in which probability of hitting $G$ from state $s$ is $\frac{1-r^{s-1}}{1-r^{G-1}}$. Then, the probability of hitting state $1$ from state $s$ becomes
		\begin{align}
			\Pr( H_1 | s^{\rho}_1 = s) = 1 - \frac{1-r^{s-1}}{1-r^{G-1}} = \frac{r^{s-1} - r^{G-1}}{1-r^{G-1}} . \nonumber
		\end{align}
			
		\textbf{(ii)} 
		Since action $C$ might be selected for more than once in a round, we have $N_{a}(\rho) \geq n_{a}(1)+n_{a}(2)+\dots+n_{a}(\rho)$. This holds because in the initialization of GETBE, each action is selected once. Basically, we derive the lower bounds for $N_{a}(\rho+1)$, but these lower bounds also hold for $N_{a}(\rho)$ because of the way GETBE is initialized.
		For a set of rounds $\rho \in {\cal T} \subset \{ 1, \ldots, T \}$, $n_a(\rho)$s are in general not identically distributed. But if $\hat{\pi}_\rho$ is same for all rounds $\rho \in {\cal T}$, then $n_a(\rho)$s are identically distributed. 
			
		First, assume that $N_{1}(\rho) = k$, $0 \leq k \leq \rho$. Then, the probability that action $C$ is selected at least once in each of these $k$ rounds is $p_{C, 1}$. Let $j_i$ denote the index of the round in which the estimated optimal policy is $\pi^{tr}_1$ for the $i$th time. The sequence of Bernoulli random variables $n_C(j_i)$, $i=1,\ldots, k$ are independent and identically distributed. 
		Hence, the Hoeffding bound given in Appendix \ref{App:AppendixA} can be used to upper-bound the deviation probability of sum of these random variables from the expected sum. Since the estimated optimal policy will be $\pi^{tr}_0$ for the remaining $\rho - k$ rounds, the number of times action $F$ is selected in all of these rounds will be at most $\lceil D(\rho) \rceil$.
		Therefore, the probability of taking action $C$ is zero for at most $\lceil D(\rho) \rceil$ rounds. Let $\rho_D := \rho - \lceil D(\rho) \rceil$ and 
		$N^\prime_C(\rho)$ denote the sum of $\rho$ random variables that are drawn from an independent identically distributed Bernoulli random process with parameter $p_{C, 1}$. 
		Then, 
		\begin{align} \label{eqn:LowerNc}
			& N_{C}(\rho)  \geq \rho - k - \lceil D(\rho) \rceil + \sum_{i=1}^{k} n_C(j_i) \nonumber \\
			& \geq \sum_{i=1}^{\rho- \lceil D(\rho)\rceil} n_C(j_i) = N\textprime_C(\rho - \lceil D(\rho) \rceil) \nonumber \\
			&  =  N\textprime_C(\rho_D) .
		\end{align}
		According to the Hoeffding bound in Appendix \ref{App:AppendixA}, we have for $z > 0$
		\begin{align}
			& \Pr\left( N\textprime_C(\rho_D) - \mathbb{E} ( N\textprime_C(\rho_D) ) \leq - z \right) \leq e^{-2z^2/\rho_D} . \nonumber
		\end{align}
		When $z = \sqrt{\rho_D \log\rho}$ the above bound becomes
		\begin{align}
			& \Pr ( N\textprime_C(\rho_D) \leq \mathbb{E} ( N\textprime_C(\rho_D) ) - \sqrt{\rho_D \log\rho} ) \leq \frac{1}{\rho^2} \nonumber \\
			\Rightarrow & \Pr ( N\textprime_C(\rho_D) \leq p_{C, 1}(\rho - \lceil D(\rho) \rceil) - \nonumber \\ 
			& \sqrt{(\rho - \lceil D(\rho) \rceil) \log\rho} ) \leq \frac{1}{\rho^2} . \nonumber
		\end{align}
		Then, by using \eqref{eqn:LowerNc} we obtain
		\begin{align} 
			\Pr ( N_C(\rho) \leq p_{C, 1}(\rho - \lceil D(\rho) \rceil) - \nonumber \\ \sqrt{(\rho - \lceil D(\rho) \rceil) \log\rho} ) \leq \frac{1}{\rho^2} . \nonumber
		\end{align}
		Since $\rho^{\prime}_C$ is the first round in which $0.5p_{C, 1}\rho - p_{C, 1} \lceil D(\rho) \rceil - \sqrt{\rho_D \log\rho}$ becomes positive, on or after $\rho^{\prime}_C$, we have $p_{C, 1}\rho_D - \sqrt{\rho_D \log\rho} > 0.5p_{C, 1}\rho$. Therefore, we replace $p_{C, 1}\rho_D - \sqrt{\rho_D \log\rho}$ with $0.5p_{C, 1}\rho$ and then
		\begin{align} \label{eqn:fc}
			\Pr ( N_C(\rho) \leq 0.5p_{C, 1}\rho ) \leq \frac{1}{\rho^2}, ~ \text{for} ~ \rho \geq \rho^{\prime}_C \nonumber
		\end{align}
		which is equivalent to
		\begin{align}
			\Pr ( N_C(\rho) \leq f_C(\rho) ) \leq \frac{1}{\rho^2}, ~ \text{for} ~ \rho \geq \rho^{\prime}_C.
		\end{align}
			
		Again, assume that $N_{1}(\rho) = k$. Then, the probability of selecting action $F$ is $p_{F, 1}$ in each of these $k$ rounds. 
		Let ${\cal R}$ denote the set of the remaining $\rho-k$ rounds. For a round $\rho_r \in {\cal R}$, action $F$ is selected only if $N_F(\rho_r) \leq D(\rho_r)$.
		Among the rounds in ${\cal R}$, the number of rounds in which action $F$ is selected is bounded below by 
		$\min \{ \rho-k,  \lceil D(\rho-k) \rceil \}$.
		Then, $n_F(j_i)$, $i=1,2,\ldots$ is a sequence of i.i.d. Bernoulli random variables with parameter $p_{F,1}$. From the argument above, we obtain
		\begin{align}
			N_{F}(\rho) & \geq \min\{\rho - k, \lceil D(\rho-k) \rceil\} + \sum_{i=1}^{k} n_F(j_i) \nonumber \\
			& \geq \sum_{i=1}^{k + \min\{\rho - k, \lceil D(\rho-k) \rceil\}} n_F(j_i) \nonumber
		\end{align}
		When $\min\{\rho - k, \lceil D(\rho-k) \rceil\} = \rho - k$, we have 
		\begin{align}
			N_{F}(\rho) & \geq 
			\sum_{i=1}^{ \rho }  n_F(j_i) \geq 
			\sum_{i=1}^{ \lceil D(\rho) \rceil }  n_F(j_i), 
			\text{ for } \rho \geq \rho^\prime_F  .  \notag
		\end{align}
		When $\min\{\rho - k, \lceil D(\rho-k) \rceil\} = \lceil D(\rho-k) \rceil$, we have 
		\begin{align}
			N_{F}(\rho) & \geq  \sum_{i=1}^{ k+ \lceil D(\rho-k) \rceil }  n_F(j_i)   .  \notag
		\end{align}	
			
		Next, we will show that $D(\rho-k)+k \geq D(\rho)$ when $\rho$ is sufficiently large. 
		First, $\min\{\rho - k, \lceil D(\rho-k) \rceil\} = \lceil D(\rho-k) \rceil$ implies that 
		\begin{align}
			\rho \geq \lceil D(\rho-k) \rceil +k \geq D(\rho-k) + k . \label{eqn:condA}
		\end{align}
		Also, $D(\rho-k)+k \geq D(\rho)$ should imply that 
		\begin{align}
			& D(\rho) - D(\rho-k) \leq k \Rightarrow
			\gamma \log (\rho/(\rho-k)) \leq k \Rightarrow \nonumber \\
			& \rho/(\rho-k) \leq e^{k/\gamma} \Rightarrow \rho \geq \frac{ke^{k/\gamma}}{e^{k/\gamma}-1} \label{eqn:condB}
		\end{align}
		Using the results in \eqref{eqn:condA} and \eqref{eqn:condB}, we conclude that $D(\rho-k)+k \geq D(\rho)$ holds when 
		\begin{align}
			D(\rho-k) + k \geq \frac{ke^{k/\gamma}}{e^{k/\gamma}-1} . \label{eqn:condC}
		\end{align}
		By setting $D(\rho) = \gamma \log \rho $ and manipulating \eqref{eqn:condC} we get
		\begin{align}
			& \gamma \log (\rho-k) + k \geq \frac{ke^{k/\gamma}}{e^{k/\gamma}-1} \Rightarrow \nonumber \\
			& \gamma \log (\rho-k) \geq k(\frac{e^{k/\gamma}}{e^{k/\gamma}-1}-1) \Rightarrow \nonumber \\
			& \log (\rho-k) \geq \frac{k/\gamma}{e^{k/\gamma}-1} \Rightarrow \nonumber \\
			& \rho-k \geq e^{\frac{k/\gamma}{e^{k/\gamma}-1}} \Rightarrow \nonumber \\
			& \rho \geq k +  e^{\frac{k/\gamma}{e^{k/\gamma}-1}} . \label{eqn:condD}
		\end{align}
		First, we evaluate the term $h(k) := e^{\frac{k/\gamma}{e^{k/\gamma}-1}}$. We will show that $h(k) \in [1,e]$ for all $k \in \mathbb{R}_{+}$. By applying L'Hopital's rule we get
		\begin{align} 
			\lim_{k \rightarrow 0} \frac{k/\gamma}{e^{k/\gamma}-1}
			= \lim_{k \rightarrow 0} \frac{1/\gamma}{(1/\gamma) e^{k/\gamma}} 
			=1  \notag
	%			\label{eqn:condE}
		\end{align}
		and 
		\begin{align} 
			\lim_{k \rightarrow \infty} \frac{k/\gamma}{e^{k/\gamma}-1}
			=0 \notag 
	%			\label{eqn:condF}.
		\end{align}	
			%		\com{Give reference to the theorem 
			%			$\lim_{k \rightarrow 0} e^{\frac{k/\gamma}{e^{k/\gamma}-1}} =
			%			e^{ \lim_{k \rightarrow 0} \frac{k/\gamma}{e^{k/\gamma}-1} } $.
			%			The references are commented here and they all assert that it is enough to have a continuous function.
			% These websites has proved this: $http://tutorial.math.lamar.edu/Classes/CalcI/LimitProofs.aspx#Extras_Limit_LimComp
			%		https://www.math.wisc.edu/~angenent/Free-Lecture-Notes/free221.pdf
			%		}
		These two conditions and using the fact that exponential function is continuous we conclude that 
		\begin{align} 
			\lim_{k \rightarrow 0} e^{ \frac{k/\gamma}{e^{k/\gamma}-1} } = e^{ \lim_{k \rightarrow 0} \frac{k/\gamma}{e^{k/\gamma}-1} } = e \notag
		\end{align}	
		and
		\begin{align} 
			\lim_{k \rightarrow \infty} e^{ \frac{k/\gamma}{e^{k/\gamma}-1} } =  e^{ \lim_{k \rightarrow \infty} \frac{k/\gamma}{e^{k/\gamma}-1} } = 1 . \notag
		\end{align}
		Next, we will show that $e^{ \frac{k/\gamma}{e^{k/\gamma}-1} }$ is decreasing in $k$. 
		Since this is a monotonically increasing function of $\frac{k/\gamma}{e^{k/\gamma}-1}$, it is sufficient to show that $\frac{k/\gamma}{e^{k/\gamma}-1}$ is decreasing in $k$. 
		%		\com{This is always true since if we have $f(g(x))$ where $g(x+k) < g(x), ~ k>0$ and $f(y)$ is an increasing function of $y$, then we have $f(g(x)) > f(g(x+k)), ~ k>0$ which makes $f(g(x))$ becomes a decreasing function of $x$.}
		We have
		\begin{align}
			\frac{d}{dk} \frac{k/\gamma}{e^{k/\gamma}-1} & = \frac{\frac{1}{\gamma}(e^{k/\gamma}-1)-\frac{k}{\gamma}(e^{k/\gamma}/\gamma)}{(e^{k/\gamma}-1)^2} \nonumber
		\end{align}
		The denominator is always positive for $k >0$. Therefore, we only consider the numerator and write it as
		\begin{align} \label{eqn:sign}
			\frac{1}{\gamma}(e^{k/\gamma}-1)-\frac{k}{\gamma}(e^{k/\gamma}/\gamma) = \frac{(\gamma-k) e^{k/\gamma} - \gamma}{\gamma^2}. \notag
		\end{align}
		As the denominator is positive, we only need to show that $(\gamma-k) e^{k/\gamma} - \gamma$ is always negative. The derivative of the above expression is $-(k/\gamma)e^{k/\gamma}$, which is negative for $k > 0$. We also have $(\gamma-k) e^{k/\gamma} - \gamma = 0$ at $k = 0$.
		These two conditions imply that $(\gamma-k) e^{k/\gamma} - \gamma$ is always negative for $k >0$, by which we conclude that $e^{ \frac{k/\gamma}{e^{k/\gamma}-1} }$ is decreasing in $k$. 
		Hence, we have 
		\begin{equation}
			1 \leq e^{\frac{k/\gamma}{e^{k/\gamma}-1}} \leq e \notag
		\end{equation}
		This implies that $k + e^{\frac{k/\gamma}{e^{k/\gamma}-1}} \leq k + e$.
		Hence \eqref{eqn:condD} holds when $\rho \geq k+e$.
		This implies that when $k \leq \rho-e$, we have
		\begin{equation}
			\sum_{i=1}^{ k+ \lceil D(\rho-k) \rceil }  n_F(j_i) \geq \sum_{i=1}^{ \lceil D(\rho) \rceil }  n_F(j_i) .  \notag
		\end{equation}
		The only cases that are left out are $k=\rho$, $k=\rho-1$ and $k=\rho-2$. But we know from the definition of $\rho'_F$ that for $\rho \geq \rho'_F$, $\rho-2- \lceil D(\rho) \rceil$ is positive. Hence for these cases we also have
		\begin{equation}
			\sum_{i=1}^{ k+ \lceil D(\rho-k) \rceil }  n_F(j_i) \geq \sum_{i=1}^{ \rho-2 }  n_F(j_i) 
			\geq \sum_{i=1}^{\lceil D(\rho)\rceil } n_F(j_i) .  \notag
		\end{equation}
			
		Let $N^\prime_F(\rho)$ denote the sum over $n_F(j_i)$ for $\rho$ rounds. From all of the cases we derived above, we obtain
		\begin{align} \label{eqn:31}
			N_{F}(\rho) \geq \sum_{i=1}^{\lceil D(\rho) \rceil} n_F(j_i) = N\textprime_F(\lceil D(\rho) \rceil) \text{ for } \rho \geq \rho'_F 
		\end{align}
		Now, by using Hoeffding bound we have
		\begin{align}
			\Pr\left(N\textprime_F(\lceil D(\rho) \rceil) - \mathbb{E} ( N\textprime_F(\lceil D(\rho) \rceil) ) \leq -z \right) \leq e^{-2z^2/\lceil D(\rho) \rceil} \nonumber
		\end{align}
		and if $z = \sqrt{\lceil D(\rho) \rceil\log\rho}$ then,
		\begin{align}
			& \Pr ( N\textprime_{F}(\lceil D(\rho) \rceil) < \mathbb{E} ( N\textprime_{F}(\lceil D(\rho) \rceil) ) - \sqrt{\lceil D(\rho) \rceil\log\rho} ) \leq \frac{1}{\rho^2} \nonumber \\
			& \Pr ( N\textprime_{F}(\lceil D(\rho) \rceil) < p_{F, 1}\lceil D(\rho) \rceil - \sqrt{\lceil D(\rho) \rceil\log\rho} ) \leq \frac{1}{\rho^2} . \nonumber
		\end{align}
		By using \eqref{eqn:31}, we get
		\begin{align}
			\Pr ( N_{F}(\rho) < p_{F, 1}\lceil D(\rho) \rceil - \sqrt{\lceil D(\rho) \rceil\log\rho} ) \leq \frac{1}{\rho^2}. \label{eqn:NFbound} 
		\end{align}
%		\rev{
%		The term $p_{F, 1}\lceil D(\rho) \rceil - \sqrt{\lceil D(\rho) \rceil\log\rho}$ is positive by the choice of $D(\rho)$ made in the statement of the lemma, since 
%		\begin{align}
%			& D(\rho) > \frac{1} { p^2_{F, 1} } \log \rho \Rightarrow \nonumber \\
%			& \lceil D(\rho) \rceil > \frac{1} { p^2_{F, 1} } \log \rho \Rightarrow \nonumber \\
%			& \sqrt{\lceil D(\rho) \rceil} - \frac{1} { p_{F, 1} } \sqrt{\log\rho} > 0 \Rightarrow \nonumber \\
%			& p_{F, 1}\lceil D(\rho) \rceil - \sqrt{\lceil D(\rho) \rceil\log\rho} > 0 \nonumber
%		\end{align}
%		}\com{Nima, do we need the above steps highlighted in blue? Or can we remove them?}
		Then, by using $D(\rho) = \gamma\log\rho, ~ \gamma > 1/p^2_{F, 1}$, we have 
		\begin{align}
			& p_{F, 1} \lceil D(\rho) \rceil - \sqrt{\lceil D(\rho) \rceil\log\rho} \nonumber \\
			& = p_{F, 1} \lceil \gamma\log\rho \rceil - \sqrt{\lceil \gamma\log\rho \rceil\log\rho} \nonumber \\
			& \geq p_{F, 1} \gamma\log\rho - \sqrt{ (\gamma\log\rho+1) \log\rho} \nonumber \\
			& = p_{F, 1} \gamma \log\rho - \sqrt{ \gamma\log^2\rho+\log\rho} \nonumber \\
			& \geq p_{F, 1} \gamma \log\rho - \sqrt{ \gamma\log^2\rho} - \sqrt{\log\rho} \label{eqn:concave} \\
			& = (p_{F, 1} \gamma - \sqrt{\gamma}) \log\rho - \sqrt{\log\rho} \label{eqn:gammacond}
		\end{align}
		where \eqref{eqn:concave} occurs due to the subadditivity\footnote{For $a, b >0$ we have $\sqrt{a}+\sqrt{b} > \sqrt{a+b}$ since $(\sqrt{a}+\sqrt{b})^2 > a+b$.} of the square root.
		Next, we will show that \eqref{eqn:gammacond} becomes positive when $\rho$ is large enough. To do this, we first show that the first term in \eqref{eqn:gammacond} is always positive. This is proven by observing that 
		\begin{align}
			\gamma > 1/p^2_{F, 1}  \Rightarrow p_{F, 1} \sqrt{\gamma} - 1 > 0 \Rightarrow p_{F, 1} \gamma - \sqrt{\gamma} > 0 . \label{eqn:positivterm}
		\end{align}
		Since $\log \rho$ increases at a higher rate than $\sqrt{\log \rho}$, it can be shown that 
		$0.5 (p_{F, 1}\gamma - \sqrt{\gamma}) \log \rho - \sqrt{\log\rho}$ will always increase after some round. Since $\lim_{\rho \rightarrow \infty} 0.5 (p_{F, 1}\gamma - \sqrt{\gamma}) \log \rho - \sqrt{\log\rho} = \infty$, this term is expected to be positive after some round. From the statement of the lemma, it is known that $\rho'_F$ is greater than or equal to this round. 
		Therefore, for $\rho \geq \rho^{\prime}_F$, $(p_{F, 1}\gamma - \sqrt{\gamma})\log\rho - \sqrt{\log\rho} > 0.5(p_{F, 1}\gamma - \sqrt{\gamma}) \log \rho$. 
		Using this and \eqref{eqn:gammacond}, we obtain
		\begin{align}
			& p_{F, 1} \lceil D(\rho) \rceil - \sqrt{\lceil D(\rho) \rceil\log\rho} \notag \\
			& \geq (p_{F, 1} \gamma - \sqrt{\gamma}) \log\rho - \sqrt{\log\rho} \notag \\
			& \geq 0.5(p_{F, 1}\gamma - \sqrt{\gamma}) \log \rho. \notag
		\end{align}
		Then, we use this result and \eqref{eqn:NFbound} to get
		\begin{align} 
			& \Pr ( N_{F}(\rho) \leq 0.5(p_{F, 1}\gamma - \sqrt{\gamma}) \log \rho ) \leq \frac{1}{\rho^2}, \text{for } \rho \geq \rho^{\prime}_F \nonumber
		\end{align}
		which is equivalent to 
		\begin{align} \label{eqn:ff}
			\Pr ( N_{F}(\rho) \leq f_F(\rho) ) \leq \frac{1}{\rho^2}, \text{for } \rho \geq \rho^{\prime}_F. 
		\end{align}
	\end{proof}
	\comment{
	\begin{corollary} \label{cor:probs}
		When the initial state distribution is uniform over $\tilde{{\cal S}}$, we have  
		\begin{align}
			p_{C, 1} & = \frac{G-2}{G-1} \nonumber \\
			p_{F, 1} & = 
			\begin{cases}
				\frac{G}{2(G-1)} & \text{if } r=1 \\
				\frac{ (G-1)(1-r)-r+r^{G}}{(G-1)(1-r)(1-r^{G-1}) } & \text{if } r\neq 1. \nonumber
			\end{cases}
		\end{align}
	\end{corollary}
	\begin{proof}
		For action C, by using the result found in the first part of Lemma \ref{lemma:MinNumPBinf} we have
		\begin{equation}
			p_{C, 1} = 1 - \frac{1}{G-1} = \frac{G-2}{G-1}. \nonumber
		\end{equation}
		For action F, we consider two cases. When $r = 1$, by using the uniform initial distribution and the result found in the first part of lemma \ref{lemma:MinNumPBinf} we have
		\begin{align}
			p_{F,1} 
			& = \left (1 + \frac{G-2}{G-1} + \frac{G-3}{G-1} + \dots + \frac{1}{G-1} \right) \frac{1}{G-1} \notag \\
			& = \left( \frac{ \frac{(G-1)G}{2} }{ G-1 } \right) \frac{1}{G-1} = \frac{G}{2(G-1)}.   \notag
		\end{align}
		When $r \neq 1$, by using the uniform initial distribution and the result found in the first part of lemma \ref{lemma:MinNumPBinf} we have
		\begin{align}
			p_{F,1} 
			& =  \left(1 + \frac{r-r^{G-1}}{1-r^{G-1}} + \dots + \frac{r^{G-2}-r^{G-1}}{1-r^{G-1}} \right) \frac{1}{G-1} \notag \\ 
			& = \frac{\sum_{i=0}^{G-2} r^i - (G-1)r^{G-1}}{(1-r^{G-1})(G-1)} \notag \\
			& = \frac{\frac{1-r^{G-1}}{1-r} - (G-1)r^{G-1}}{(1-r^{G-1})(G-1)} \notag \\
			& = \frac{(G-1)r^{G} - Gr^{G-1} + 1}{ (G-1) (r^G - r^{G-1} - r + 1) }  \notag
		\end{align}
	\end{proof}
	}
	The (expected) regret given in \eqref{eqn:regretreduced} can be decomposed into two parts: (i) regret in rounds in which the suboptimal policy is selected, (ii) regret in rounds in which the optimal policy is selected and GETBE explores. Let $\text{IR}(T)$ denote the number of rounds by round $T$ in which the suboptimal policy is selected.
	The first part of the regret is upper bounded by $\mathbb{E} (\text{IR}(T))$, since the reward in a round can be either $0$ or $1$. Similarly, the second part of the regret is upper bounded by the number of explorations when the optimal policy is $\pi^{tr}_0$. When the optimal policy is $\pi^{tr}_1$, exploration will only be performed when the suboptimal policy is selected. Hence, there is no additional regret due to explorations, since all the regret is accounted for in the first part of the regret.
%	\com{I think the same issue happens here. When we choose suboptimal and play it the regret incurred will be $\Delta_{\max}$, which will be different than the regret incurred when we choose suboptimal but explore. Because of this I removed $\Delta_{\max}$ from the regret bound.}
	
	Let $A_{\rho}$ denote the event that the suboptimal policy is selected in round $\rho$. 
	Let
	\begin{align}
		C_{\rho} 
		:= \{ |p^C - \hat{p}^{C}_\rho| \geq \delta/\sqrt{2}  \} 
		\cup   
		\{ |p^F - \hat{p}^{F}_\rho| \geq \delta/\sqrt{2}  \} . \notag 
	\end{align}
	It can be shown that on event $C^c_{\rho}$ the Euclidian distance between $(p^C, p^F)$ and $(\hat{p}^{C}_\rho,\hat{p}^{F}_\rho)$ is less than $\delta$. This implies that on event $C^c_{\rho}$, the optimal policy is selected. Therefore, $C_{\rho}$ contains the event that the optimal policy is not selected. 
	Using the linearity of expectation and the union bound, we obtain
	\begin{align} \label{eqn:IRbound}
		\mathbb{E} [\text{IR}(T)]
		& =
		\mathbb{E} [ \sum_{\rho=1}^{T} \mathbb{I}(A_\rho) ] \nonumber \\
		& \leq \sum\limits_{\rho=1}^{T} { \sum\limits_{a\in \{ F, C \}} {\Pr\left( |p^{\text{a}} - \hat{p}^{\text{a}}_\rho| \geq \delta/\sqrt{2} \right)} }. 
	\end{align}
	
	Let $\mathbb{I}^{\text{exp}}_\rho$ be the indicator function of the event that GETBE explores.
	By the above discussion we have
	\begin{align}
		\mathbb{E} [\text{Reg}(T) | \pi^* = \pi^{tr}_1] & \leq   
		\mathbb{E} [\text{IR}(T) ]
		\label{eqn:boundregret1} \\
		\mathbb{E} [\text{Reg}(T) | \pi^* = \pi^{tr}_0] & \leq \Delta_{\max} \mathbb{E} [\text{IR}(T) ]  \notag \\ 
		& + 
		\mathbb{E} [ \sum_{\rho=1}^{T} \mathbb{I}^{\text{exp}}_{\rho}]  \label{eqn:boundregret2} .
	\end{align}
%	\com{We don't need to remove $\Delta_{\max}$ from $\mathbb{E} [\text{Reg}(T) | \pi^* = \pi^{tr}_0]$ because we know that we are going to play according to policy $\pi^{tr}_1$ when we make a suboptimal choice (we will not explore).}
	
	Next, we bound the expected regret of GETBE for the GRBP using \eqref{eqn:boundregret1} and \eqref{eqn:boundregret2}. 

	\begin{theorem} \label{thm:GRBPreg}
		Let $x_1 := \left(1+\sqrt{(24 p_{F, 1} / \delta^2)+1} \right) / 2 p_{F, 1}$. 
		Assume that the control function is 
		\begin{align}
			D(\rho) = \gamma\log\rho \text{ where } \gamma > \max\{(x_1)^2, \frac{1}{(p_{F, 1})^2}\}. \nonumber
		\end{align}
		Let $\rho^{\prime\prime}$ be the first round in which $\delta^2 \geq \frac{3 \log \rho}{f_a(\rho)}$ for both actions, $\rho^\prime := \max\{\rho^\prime_C, \rho^\prime_F, \rho^{\prime\prime}\}$, and 
		\begin{align}
			w := 4\rho^{\prime} + \frac{2\pi^2}{3}( 1 + \frac{1}{1-e^{-\delta^2}} ) \nonumber
		\end{align}
		Then, the regret of GETBE is bounded by
		\begin{equation}
			\mathbb{E} [\text{Reg}(T) | \pi^* = \pi^{tr}_1] \leq w \nonumber
		\end{equation}
		and 
		\begin{equation}
			\mathbb{E} [\text{Reg}(T) | \pi^* = \pi^{tr}_0] \leq \lceil D(T) \rceil +  w \Delta_{\max} . \nonumber
		\end{equation} 
	\end{theorem}
	\begin{proof}
		First, we bound $\mathbb{E} [\text{IR}(T)]$. For this, we replace the order of summations in \eqref{eqn:IRbound} and we have
		\begin{align} \label{eqn:ErrDef1}
			\mathbb{E} [\text{IR}(T)] \leq \sum\limits_{a\in \{ F, C \}} \sum\limits_{\rho = 1}^{T} \Pr\left( |p^{\text{a}} - \hat{p}^{\text{a}}_\rho| \geq \delta/\sqrt{2} \right).
		\end{align}
		Let $N_{F}^{*}(\rho) := N_{F}^{G}(\rho)$ and $N_{C}^{*}(\rho) := N_{C}^{u}(\rho)$.
		By using the law of total probability and Hoeffding inequality, we obtain for $a \in \{ F, C \}$
		\begin{align} \label{eqn:CHres1}
			& \sum\limits_{\rho=1}^{T} \Pr\left(|p^{a} - \hat{p}^{a}_\rho|\geq \frac{\delta}{\sqrt{2}} \right) \nonumber \\ 
			& = \sum\limits_{\rho=1}^{T} \sum\limits_{n=1}^{\infty} {\Pr\left( |p^{a} - \frac{N_a^{*}(\rho)}{N_{a}(\rho)}| \geq \frac{\delta}{\sqrt{2}}|N_{a}(\rho)=n \right) \Pr\left( N_{a}(\rho)=n \right)} \nonumber \\
			& = \sum\limits_{\rho=1}^{T} \sum\limits_{n=1}^{\infty} {\Pr\left( |n p^{a} - N_{a}^{*}(\rho)| \geq n\frac{\delta}{\sqrt{2}} \right) \Pr\left( N_{a}(\rho)=n \right)} \nonumber \\
			& \leq \sum\limits_{\rho=1}^{T} \sum\limits_{n=1}^{\infty} {2e^{-2\frac{(n\delta/\sqrt{2})^2}{n}} \Pr\left( N_{a}(\rho)=n \right)} \nonumber \\
			& = \sum\limits_{\rho=1}^{T}\sum\limits_{n=1}^{\infty} {2e^{-n\delta^2} \Pr\left( N_{a}(\rho)=n \right)} .
		\end{align}
%		\com{There is a problem with the above proof. $N_C(\rho)$ can be greater than $\rho$ in practice. The law of total probability does not apply exactly here. How do you overcome this? Shall we assume that GETBE updates the probability of action $C$ only once in each round?}
		For each action, we use the result of Lemma \ref{lemma:MinNumPBinf} and divide the summation in \eqref{eqn:CHres1} into two summations. Note that the bounds on $N_a(\rho)$ given in Lemma \ref{lemma:MinNumPBinf} hold when $\rho \geq \rho^{\prime} \geq \max\{ \rho^{\prime}_C, \rho^{\prime}_F\}$. Therefore, we have
		\begin{align} \label{eqn:Sep1}
			& \sum\limits_{\rho=1}^{T} \sum\limits_{n=1}^{\infty} {2e^{-n\delta^2} \Pr\left( N_{a}(\rho)=n \right)} = \nonumber \\
			& \sum\limits_{\rho=1}^{\rho^{\prime}-1} \sum\limits_{n=1}^{\infty} {2e^{-n\delta^2} \Pr\left( N_{a}(\rho)=n \right)} + \nonumber \\ 
			& \sum\limits_{\rho=\rho^{\prime}}^{T} \sum\limits_{n=1}^{\infty} { 2e^{-n\delta^2} \Pr\left( N_{a}(\rho)=n \right) } \nonumber \\
			& = K^{\prime} + \sum\limits_{\rho=\rho^{\prime}}^{T} \sum\limits_{n=1}^{\infty} { 2e^{-n\delta^2} \Pr\left( N_{a}(\rho)=n \right) }
		\end{align}
		where $K^{\prime} = \sum\limits_{\rho=1}^{\rho^{\prime}-1} \sum\limits_{n=1}^{\infty} {2e^{-n\delta^2} \Pr\left( N_{a}(\rho)=n \right)}$ which is finite since $\rho^\prime$ is finite. 
		Since
		\begin{align}
			\sum\limits_{n=1}^{\infty} { \Pr\left( N_{a}(\rho)=n \right)} = 1 \nonumber
		\end{align}
		and as $e^{-n\delta^2} \leq 1$, then we have
		\begin{align}
			\sum\limits_{n=1}^{\infty} {e^{-n\delta^2} \Pr\left( N_{a}(\rho)=n \right)} \leq 1. \nonumber
		\end{align}
		Therefore, an upper bound on $K^\prime$ can be given as
		 \begin{align}
			K^{\prime} & = \sum\limits_{\rho=1}^{\rho^{\prime}-1} \sum\limits_{n=1}^{\infty} {2e^{-n\delta^2} \Pr\left( N_{a}(\rho)=n \right)} \leq \sum\limits_{\rho=1}^{\rho^{\prime}-1} 2 < 2\rho^{\prime}. \label{eqn:boundonK}
		\end{align}
		We have
		\begin{align} \label{eqn:Sep2}
			& \sum\limits_{\rho=\rho^{\prime}}^{T}\sum\limits_{n=1}^{\infty} {2e^{-n\delta^2} \Pr\left( N_{a}(\rho)=n \right)} = \nonumber \\
			& \sum\limits_{\rho=\rho^{\prime}}^{T}\sum\limits_{n=1}^{f_a(\rho)} {2e^{-n\delta^2} \Pr\left( N_{a}(\rho)=n \right)} + \nonumber \\ 
			& \sum\limits_{\rho=\rho^{\prime}}^{T}\sum\limits_{n=f_a(\rho)+1}^{\infty} {2e^{-n\delta^2} \Pr\left( N_{a}(\rho)=n \right)} .
		\end{align}
		For the first summation in \eqref{eqn:Sep2}, we use \eqref{eqn:fc} and \eqref{eqn:ff} for each action as an upper bound since it is the case when $n \leq f_a(\rho)$. Therefore,
		\begin{align} \label{eqn:sum1bound}
			& \sum\limits_{\rho=\rho^{\prime}}^{T}\sum\limits_{n=1}^{f_a(\rho)} {2e^{-n\delta^2} \Pr\left( N_{a}(\rho)=n \right)} \leq \sum\limits_{\rho=1}^{T}\sum\limits_{n=1}^{f_a(\rho)} \frac{2e^{-n\delta^2}}{\rho^2} \nonumber \\
			& \leq \sum\limits_{\rho=1}^{\infty}\sum\limits_{n=1}^{\infty} \frac{2e^{-n\delta^2}}{\rho^2} = \frac{\pi^2}{3(1-e^{-\delta^2})}
		\end{align}
		For the second summation in \eqref{eqn:Sep2}, we first show that $\delta^2 \geq \frac{3 \log \rho}{f_a(\rho)}$ for each action when $\rho \geq \rho^\prime$. 
		%Thus, $\delta^2 \geq \frac{3 \log \rho}{n}$, for $n > f_a(\rho)$ for $\rho \geq \rho^\prime$.
		
		For $a=F$, we have $\delta^2 \geq \frac{3 \log \rho}{f_F(\rho)} = \frac{6}{p_{F, 1}\gamma - \sqrt{\gamma}}$ since $\gamma \geq (x_1)^2$. The proof is as follows. 
		Note that the term $p_{F, 1}\gamma - \sqrt{\gamma}$ is positive because of \eqref{eqn:positivterm}.
		In order to have $\delta^2 \geq \frac{6}{p_{F, 1}\gamma - \sqrt{\gamma}}$, we must have $p_{F, 1}\gamma - \sqrt{\gamma} - 6/\delta^2 \geq 0$. 
		This can be re-written as a second order polynomial function, which is given by
		\begin{align}
			g(x) = ax^2+bx+c \geq 0 \nonumber
		\end{align}
		where $a = p_{F, 1}$, $b = -1$,$c = -6/\delta^2$ and $x = \sqrt{\gamma}$. Since $\gamma$ is positive, we will find positive values of $x$ for which $g(x)$ is non-negative. Also, $g(x)$ is a convex function since its second derivative is $2a$, which is positive. Hence, $g(x)$ is non-negative for positive $x$'s which are greater than the largest root. The roots of $g(x)$ are given as
		\begin{align}
			x_1 = \frac{1+\sqrt{1+\frac{24p_{F, 1}}{\delta^2}}}{2p_{F, 1}}, ~ x_2 = \frac{1-\sqrt{1+\frac{24p_{F, 1}}{\delta^2}}}{2p_{F, 1}}. \nonumber
		\end{align}
		It is clear that only $x_1$ is positive. Thus, $g(x)$ is non-negative for $x = \sqrt{\gamma} \geq x_1$. Therefore, $\gamma$ has to be greater than $(x_1)^2$ so that $\delta^2 \geq \frac{6}{p_{F, 1}\gamma - \sqrt{\gamma}}$.
		
		For $a = C$ we have $\frac{3 \log \rho}{f_C(\rho)} = \frac{3 \log \rho}{0.5p_{C, 1}\rho}$. This quantity decreases as $\rho$ increases and converges to zero in the limit $\rho$ goes to infinity. Hence, this quantity becomes smaller than $\delta^2$ after some round. Hence, for $\rho \geq \rho^\prime$, we have $\delta^2 \geq \frac{3 \log \rho}{ f_{a}(\rho)}$ for both actions. Thus,
		\begin{align} 					 	\label{eqn:sum2bound}
			& \sum\limits_{\rho = \rho^{\prime}}^{T} \sum\limits_{n=f_a(\rho)+1}^{\infty} 2e^{-n\delta^2} \Pr\left( N_{a}(\rho)=n \right) \notag \\
			& \leq \sum\limits_{\rho = \rho'}^{T} 2e^{-f_a(\rho) \delta^2} \sum\limits_{n=f_a(\rho)+1}^{\infty} \Pr\left( N_{a}(\rho)=n \right) \nonumber \\
			&  \leq \sum\limits_{\rho = \rho'}^{T} 2e^{-f_a(\rho) \delta^2} 
			\leq \sum\limits_{\rho = \rho'}^{T} 2e^{-3 \log \rho} \notag \\			
			& \leq \sum\limits_{\rho=1}^{\infty} {\frac{2}{\rho^3}} \leq  \frac{\pi^2}{3}.
		\end{align}
		
		Finally, we combine the results of \eqref{eqn:sum2bound}, \eqref{eqn:sum1bound} and \eqref{eqn:boundonK} together with the result of \eqref{eqn:Sep1} and sum the final result over the two actions to get a bound for the expression in \eqref{eqn:ErrDef1}. This results in
		\begin{align} \label{eqn:ErrBound1}
			\mathbb{E} [\text{IR}(T)] & \leq 2( 2\rho^{\prime} +  \frac{\pi^2}{3(1-e^{-\delta^2})} + \frac{\pi^2}{3} ) \nonumber \\
			& = 4 \rho^\prime + \frac{2\pi^2}{3}( 1 + \frac{1}{1-e^{-\delta^2}} ) = w.
		\end{align}
		
		Assume the optimal policy is $\pi^{tr}_1$. Then, the expected number of rounds in which the suboptimal policy is selected is finite and bounded by $w$ (independent of $T$) in \eqref{eqn:ErrBound1}. 
		In this case, the exploration is done only when the suboptimal policy is selected and there will be no extra regret term due to exploration. 
		Therefore,		
		%		Assume the optimal policy is $\pi^{tr}_1$, then the expected number of rounds in which the region is suboptimal is finite and bounded by $w$ (independent of $T$) in \eqref{eqn:ErrBound1}. When the region is incorrect, the suboptimal policy is selected unless control function is operating at state $1$. As an upper-bound, we disregard the rounds in which the control function is operating at state $1$. As the result, the expected number of rounds that a suboptimal policy is selected is bounded by the expected number of rounds the region is suboptimal which is $w$. Therefore,
		\begin{align}
%		\label{eqn:regbound1}
			\mathbb{E} [\text{Reg}(T) | \pi^* = \pi^{tr}_1] & \leq 4 \rho^\prime + \frac{2\pi^2}{3}( 1 + \frac{1}{1-e^{-\delta^2}} )  \notag \\
			& = w . \notag
		\end{align}
		Assume the optimal policy is $\pi^{tr}_0$. Similar to the previous case, the expected number of rounds in which the suboptimal policy is selected is at most $w$. Since the suboptimal policy for this case is $\pi^{tr}_1$, it will always be played if it is selected (no exploration). Hence, the regret in these rounds is at most $\Delta_{\max}$. However, the learner will explore action $F$ when the optimal policy is selected. This results in additional regret. Since, the number of explorations of GETBE by round $T$ is bounded by $\lceil D(T) \rceil$, the regret that will result from explorations is also bounded by $\lceil D(T) \rceil$. Therefore,
		%		Again if the optimal policy is $\pi^{tr}_0$ the expected number of rounds the region is suboptimal is at most $w$ which is finite and independent of $T$. Therefore, the learner keeps exploring action $F$ when the region is optimal. This causes the number of rounds in which forced exploration is applied increases with $T$. Therefore, the suboptimal policy is explored at most $\lceil D(T) \rceil$. Hence, the bound on the number of rounds in which a suboptimal policy is selected is
		\begin{align}
%		\label{eqn:regbound2}
			\mathbb{E} [\text{Reg}(T) | \pi^* = \pi^{tr}_0] 
			& \leq  \lceil D(T) \rceil + w\Delta_{\max} . \notag
		\end{align}
	\end{proof}
	Theorem \ref{thm:GRBPreg} bounds the expected regret of GETBE. When $\pi^* = \pi^{tr}_1$, $\text{Reg} (T)  = O(1)$ since both actions will be selected with positive probability by the optimal policy at each round.
	When $\pi^* = \pi^{tr}_0$, $\text{Reg} (T)  = O( D(T) )$ since GETBE forces to explore action $F$ logarithmically many times to avoid getting stuck in the suboptimal policy. 
	
	\section{Numerical Results}\label{sec:numeric}
	%In this section, we illustrate the performance of GETBE in terms of the expected regret.
	%We also plot the expected regret as a function the transition probabilities.
	%While GETBE uses sample mean estimates of the transition probabilities, method like posterior sampling and upper confidence bound based indices can also be incorparated into GETBE to reflect the uncertainty in the estimated transition probabilities. We also illustrate the performance of GETBE with such update rules. 
	
	We create a synthetic medical treatment selection problem based on \cite{Isobe2013cancer}. Each state is assumed to be a stage of gastric cancer ($G = 4$, $D = 0$). The goal state is defined as at least three years of survival. Action $C$ is assumed to be chemotherapy and action $F$ is assumed to be surgery. 
	For action $C$, $p^C$ is determined by using the average survival rates for young and old groups at different stages of cancer given in \cite{Isobe2013cancer}. For each stage, the survival rate at three years is taken to be the probability of hitting $G$ by taking action $C$ continuously. With this information, we set $p^C = 0.45$. Also, the five-year survival rate of surgery given in  \cite{stomachcancer} ($29\%$) is used to set $p^F = 0.3$. 
	%Numerical result shows that $\delta \approx 0.07$ for these parameters of the system.
	
	The regrets shown in Fig. 3 and 4 correspond to different variants of GETBE, named as GETBE-SM, GETBE-PS and GETBE-UCB. Each variant updates the state transition probabilities in a different way. GETBE-SM uses the control function together with sample mean estimates of the state transition probabilities. Unlike GETBE-SM, GETBE-UCB and GETBE-PS do not use the control function. GETBE-PS uses posterior sampling from the Beta distribution \cite{thompson2012tsmab} to sample and update $p^F$ and $p^C$. GETBE-UCB adds an {\em inflation term} that is equal to  $\sqrt{\frac{2\log(N_F(\rho)+N_C(\rho))}{N_a(\rho)}}$ to the sample mean estimates of the state transition probabilities that correspond to action $a$. PS-PolSelection and UCB-PolSelection algorithms treat each policy as a super-arm, and use PS and UCB methods to select the best policy among the two threshold policies. Instead of updating the state transition probabilities, they directly update the rewards of the policies. 
	%	NoExplore algorithm computes a new policy at each round by solving the MAXPROB problem in \cite{Kolobov12MDPGD} using the sample mean estimates of the state transition probabilities. 
	
	Initial state distribution is taken to be the uniform distribution.
	Initial estimates of the transition probabilities are formed by setting $N_{F}(1) = 1, ~  N^{G}_{F}(1) \sim \text{Unif}[0,1], ~ N_{C}(1) = 1, ~  N^{u}_{C}(1) \sim \text{Unif}[0,1]$. 
	The time horizon is taken to be $5000$ rounds, and 
	%Also, the results given for the problem dependent regret bound is an average over 100 number of iterations of the program.
	the control function is set to be $D(\rho) = 15\log\rho$.
	Reported results are averaged over $200$ iterations.
	\begin{figure}[h]
		%	\begin{center}
		\hspace{-1em}
		\includegraphics[scale=0.275]{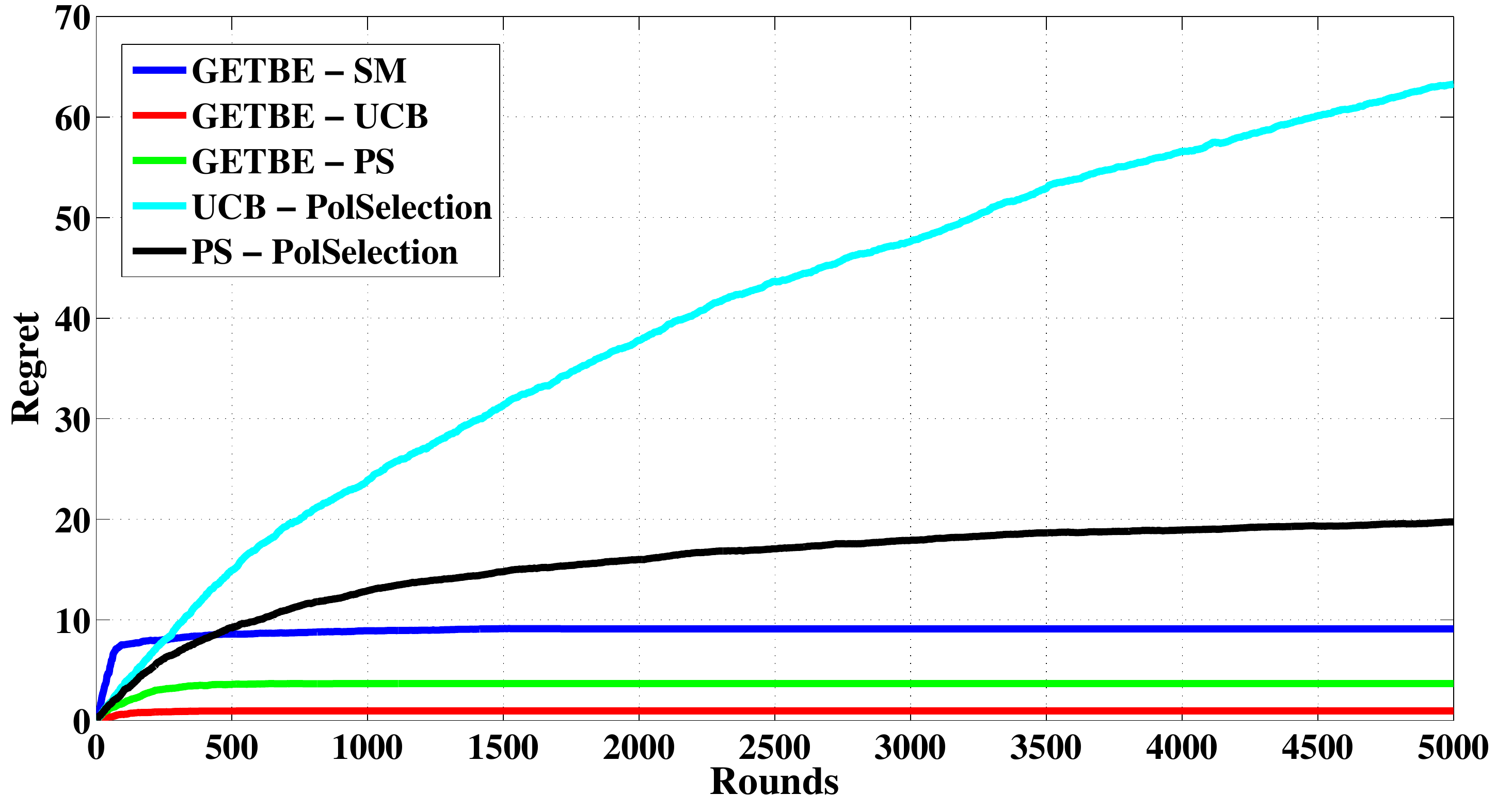}
		\caption{Regrets of GETBE and the other algorithms as a function of the number of rounds, when the transition probabilities lie in the no-exploration region.} \label{fig:regretinf}
		%	\end{center}
	\end{figure}
	
	In Fig. \ref{fig:regretinf} the regrets of GETBE and other algorithms are shown for $p^F$ and $p^C$ values given above. 
	For this case, the the optimal policy is $\pi^{tr}_1$ and all variants of GETBE achieve finite regret, as expected. 
	However, the regrets of UCB-PolSelection and PS-PolSelection increase logarithmically, since they sample each policy logarithmically many times. 
	%	For this case, the regret of NoExplore grows linearly, because it cannot update $\hat{p}^F_\rho$ when the estimated optimal policy falls into the exploration region.
	
	\begin{figure}[h]
		%	\begin{center}
		\hspace{-1em}
		\includegraphics[scale=0.275]{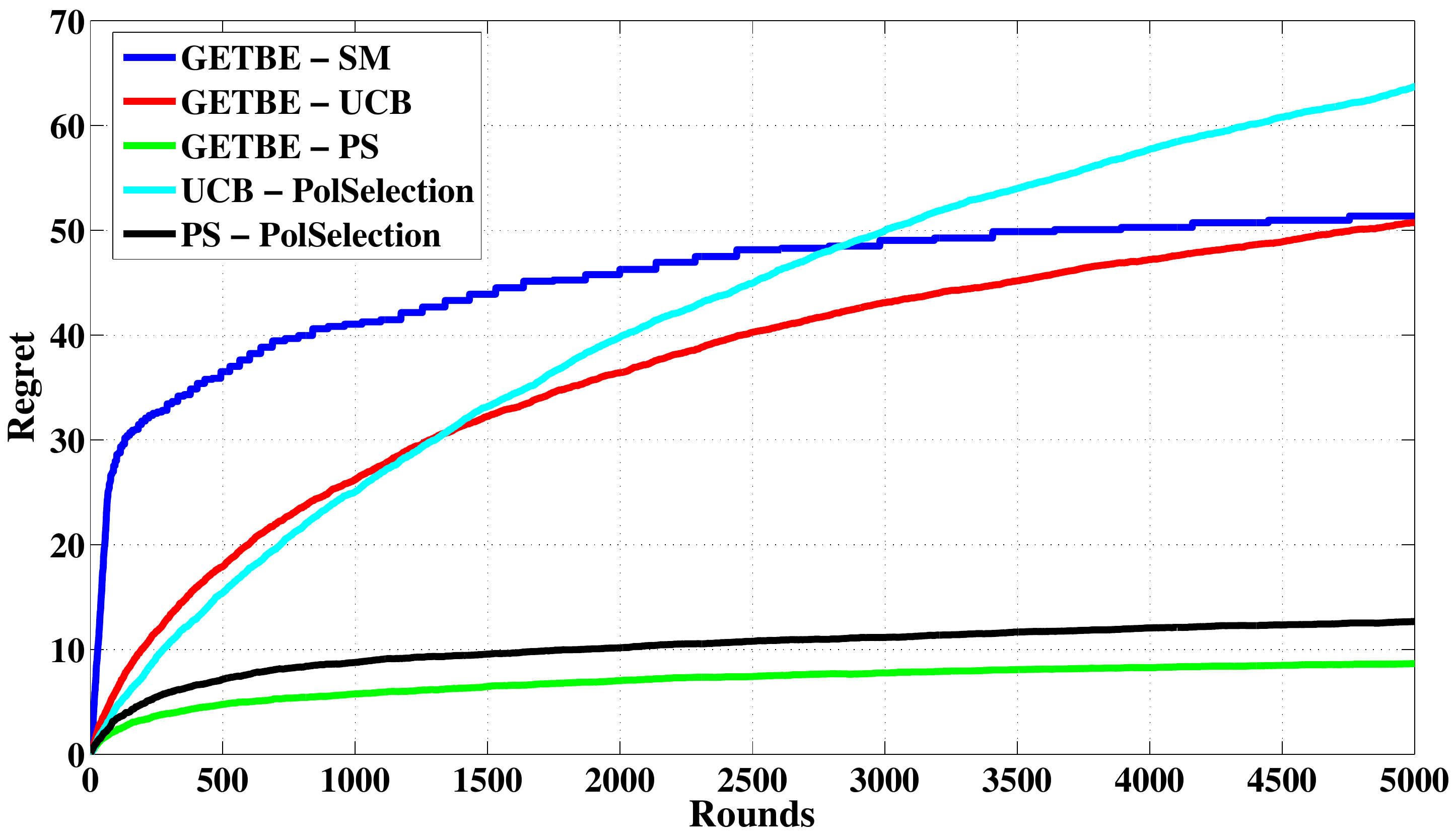}
		\caption{Regrets of GETBE and the other algorithms as a function of the number of rounds, when the transition probabilities lies in the exploration region.} \label{fig:closetobound}
		%	\end{center}
	\end{figure}
	\begin{figure}[h]
		%	\begin{center}
		\hspace{-1em}
		\includegraphics[scale=0.22]{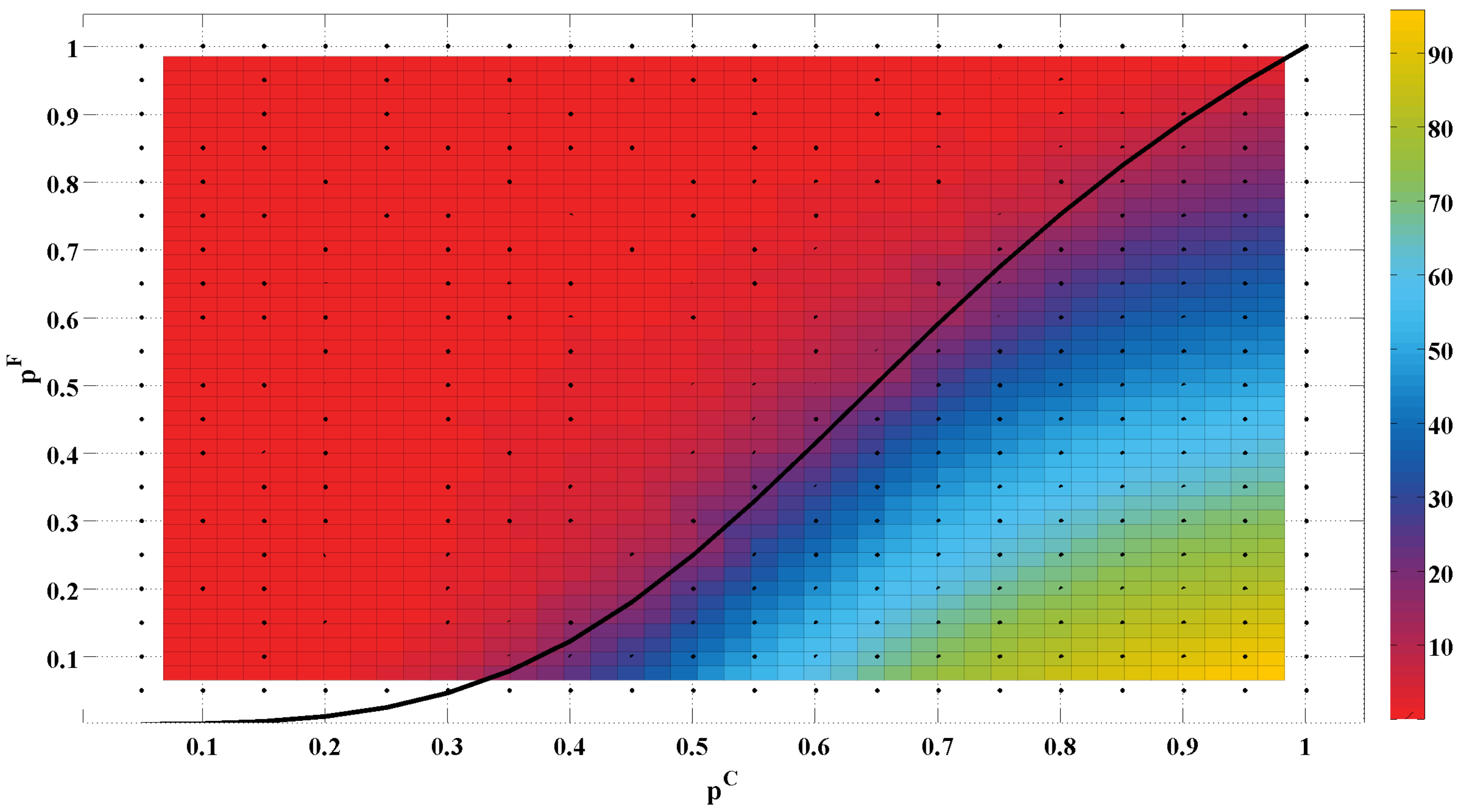}
		\caption{Regret of GETBE for different values of $p^C, p^F$. Black line shows the boundary.} \label{fig:3D}
		%	\end{center}
	\end{figure}
	Next, we set $p^C = 0.65$ and $p^F = 0.3$, in order to show how the algorithms perform when the optimal policy is $\pi^{tr}_0$.
	%Numerical result shows that $\delta \approx 0.1$ in this case.
	The result for this case is given in Fig. \ref{fig:closetobound}. As expected, the regret grows logarithmically over the rounds for all variants of GETBE, PS-PolSelection and UCB-PolSelection. GETBE-PS achieves the lowest regret for this case.
	%	On the other hand, NoExplore achieves finite regret because its transition probability estimates are stuck in the optimal region. 
	
	Fig. \ref{fig:3D} illustrates the regret of GETBE-SM as a function of $p^F$ and $p^C$ for $T=1000$. As the state transition probabilities shift from the no-exploration region to the exploration region the regret increases as expected. 
	\com{You should also show the boundary line on this figure.}
	%The boundary line is the region where the growth of regret becomes more apparent while shifting to exploration region.
	
	\section{Conclusion}\label{sec:Conclusion}
	In this paper, we introduced the Gambler's Ruin Bandit Problem. We characterized the form of the optimal policy for this problem, and then developed a learning algorithm called GETBE that operates on the GRBP to learn the optimal policy when the transition probabilities are unknown. We proved that the regret of this algorithm is either bounded (finite) or logarithmic in the number of rounds based on the region that the true transition probabilities lie in. In addition to the regret bounds, we illustrated the performance of our algorithm via numerical experiments.
	
	\bibliographystyle{IEEE}
	\bibliography{OSA}
	
	\appendix
	\section{Hoeffding Inequality} \label{App:AppendixA}
	Let $X_1, X_2, \ldots, X_n$ be random variables in range of [0, 1] and $\mathbb{E}[X_t | X_1, \ldots, X_{t-1}] = \mu$. Let $S_n = \sum_{i=1}^{n} X_i$. Then for any nonegative $z$,
	\begin{align}
		& \Pr\left( S_n \leq \mathbb{E}\left(S_n\right) - z \right) \leq \exp ( -\dfrac{2z^2}{n} ) \nonumber \\
		& \Pr\left( |S_n - \mathbb{E}\left(S_n\right)| \geq  z \right) \leq 2 \exp ( -\dfrac{2z^2}{n} ) \nonumber
	\end{align}
	
\end{document}